\documentclass{article}
\usepackage{graphicx} 
\usepackage{amsmath, amssymb}
\usepackage[utf8]{inputenc}
\usepackage[T1]{fontenc}
\usepackage[margin=1.35in]{geometry}  
\usepackage[authoryear]{natbib}
\usepackage[colorlinks=true, linkcolor=blue, citecolor=blue, urlcolor=blue]{hyperref}
\hypersetup{
  colorlinks=true,
  citecolor=blue,
  linkcolor=blue,
  urlcolor=blue
}
\usepackage{algorithm}
\usepackage{booktabs}
\usepackage{algpseudocode}
\usepackage{microtype}
\usepackage{subcaption}
\usepackage{xcolor}
\usepackage{enumitem}
\usepackage{graphicx}
\usepackage{enumitem}
\usepackage{amsmath, amsfonts, amssymb, amsthm}

\definecolor{googleblue}{HTML}{4285f4}
\definecolor{royalblue}{HTML}{4169e1}
\definecolor{greyblue}{HTML}{6699CC}
\theoremstyle{plain}
\newtheorem{theorem}{Theorem}[section]

\newtheorem{lemma}[theorem]{Lemma}

\theoremstyle{definition}

\theoremstyle{remark}

\title{Schreier Coset Graph Rewiring}

\author{
  Aryan Mishra\thanks{amishr17@umd.edu, randym@umd.edu}
  \and Randy Martinez\footnotemark[1]
  \and Lizhen Lin\thanks{\texttt{lizhen01@umd.edu}}
}
\date{}

\begin{document}

\maketitle

\begin{abstract}
   The information flow in the graph neural networks (GNNs) is fundamentally constrained by \textit{over-squashing}, where structural bottlenecks impede long range information propagation. Graph-rewiring methods, which modify graph topology,  have been extensively used to alleviate this. However, existing approaches often introduce prohibitive structural and computational bottlenecks, fail to  preserve the critical properties of original graphs, and increase the edge counts massively. We introduce a novel method \textit{\textbf{Schreier-Coset Graph Rewiring}}, a group-theoretic rewiring method that augments the input graph with a Schreier-Coset graph derived from a special linear group. Our method provides theoretical  guarantees, a graph that exhibits spectral gap and a bounded effective resistance, creating a low-resistance bypass for long-range communication. Empirical evaluations demonstrate that SCGR reduces effective resistance by $5$–$40\%$ across various learning tasks, effectively mitigating connectivity bottlenecks  while maintaining competitive accuracy.
\end{abstract}

\section{Introduction and Related Works}
Graphs are network structures that represent relationships by connecting entities (nodes) with the other entities using connection links (edges) \cite{harary1969graph, diestel2012graph, shuman2013emerging, hamilton2017inductive}. 
The ability of graphs to capture relational and structural information makes them invaluable across diverse fields such as biology, topology, recommendation systems and connected structural information \cite{battaglia2018relational, sanchez2018graph, gilmer2017neural, berg2017graph, koh2024physicochemical}.

 Graph Neural Networks (GNNs) are a specialized class of neural networks developed to process graph structural data. 
GNNs typically operate in a message passing paradigm \cite{he2023message},  where nodes iteratively exchange and aggregate information from their neighbors to update the node representations \cite{jiang2019semi, kipf2016semi, velivckovic2017graph}.  To capture long-range information or interactions within graphs, deep GNN architectures are often necessary.  However, increasing the number of layers often introduces structural and computational bottlenecks. In particular,  this results in  large amount of information from extensive neighborhoods being compressed into fixed size embeddings, known as \textbf{\textit{over-squashing}} \cite{alon2020bottleneck, arnaiz2025oversmoothing}. Over-squashing curbs the ability of GNNs to capture long range dependencies, thus degrading performance tasks that need global context. 

Various methods have been employed to address over-squashing in GNNs. In \textbf{\textit{Feature Augmentation}}, node/edge attributes are enriched with global signals. \citet{eliasof2023graph} concatenate top-\(k\) Laplacian eigenvectors to each node, thus long-range information does not propagate hop by hop. However, eigen-decomposition, which costs \(O(n^3)\) and \(O(nk)\) memory, suffers from batch inefficiency. \textbf{\textit{Graph Rewiring}} modifies the input (original) graph by \emph{\textbf{a) adding nodes}} or  \emph{\textbf{b) reconfiguring the edges}} to generate an output graph that enhances connectivity between nodes.
For instance, \citet{deac2022expander} construct expander graphs.  \citet{wilson2024cayley} use Cayley graphs to aid propagation. Some other work modify topology using properties such as curvature \cite{fesser2024mitigating}, spectral expansion \cite{karhadkar2023fosr} and effective resistance \cite{black2023understanding} to optimize information flow. 

The existing graph rewiring techniques have two limitations: \textbf{\textit{a) over-alter the edges of input graph}} and \textbf{\textit{b) introduce newer edges}}.  Expander graph rewiring methods introduce entirely new nodes sets, diverging significantly from the structure of input graph \cite{rampášek2023recipegeneralpowerfulscalable}.  Graph transformers rely on FCNNs (fully connected neural networks) and have quadratic- computation scaling. The Delaunay graphs \citep{attali2024delaunay} provide benefits such as reduced graph diameter and lower effective resistance. However, it constructs graphs solely based on node features, completely disregarding the original graph topology. 

\begin{figure*}[t]
  \centering
  \hspace*{3pt}%
\includegraphics[width=\textwidth,height=0.5\textheight,keepaspectratio]{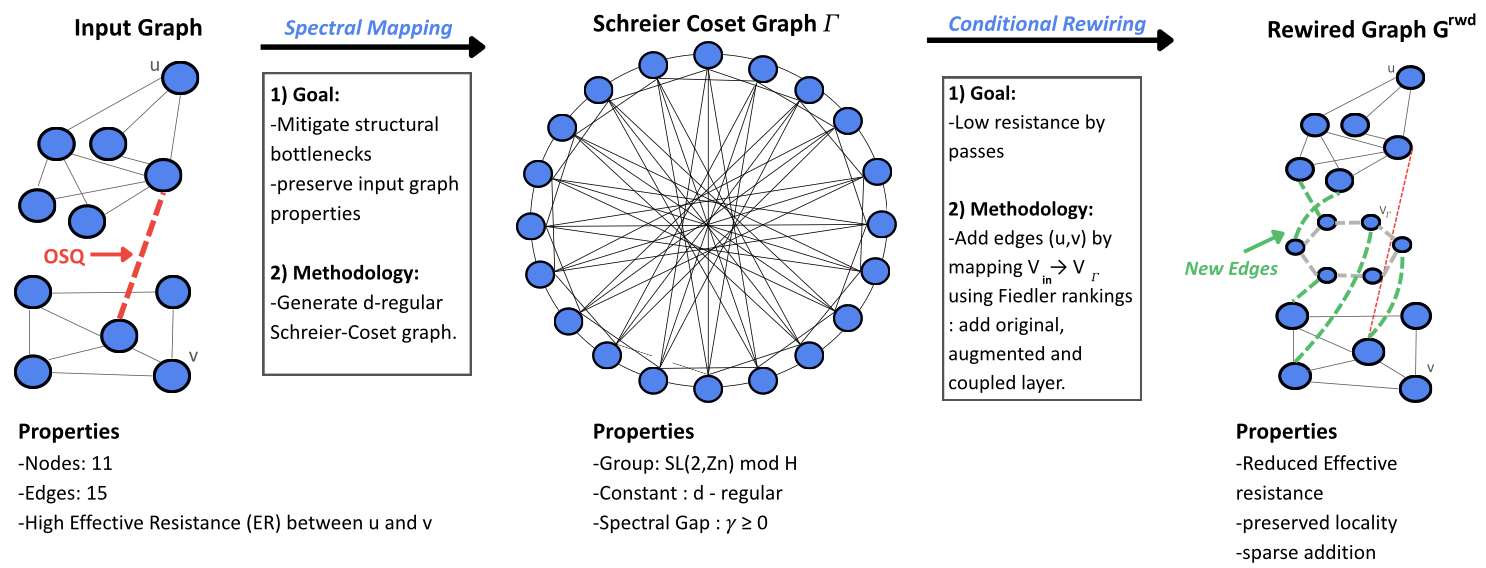}
  \caption{Illustrates the Schreier–coset graph rewiring framework, original graph locality is preserved,\(d\)-regular graph \(\Gamma\) provides constant-degree shortcuts with bounded effective resistance, align nodes via Fiedler Ranking
  (\textcolor{red}{OSQ} represents \textit{over-squashed} edge link)}
  \label{fig:Scr}
\end{figure*}

Preserving the graph locality is imperative during rewiring, since many learning tasks, such as clustering and semi-supervised learning on graphs, rely on spectral properties of input graphs to ensure accurate results. These methods often introduce many additional edges to enhance connectivity, which both increases computational cost of learning on the rewired output graph \cite{arnaiz2025oversmoothing} and the risk of over-smoothing \cite{li2018deeper,liang2023predicting,liang2023tackling}.
\\
To address over-squashing, we propose \textbf{\textit{Schreier-Coset Graph Rewiring (SCGR)}} (refer \textbf{Figure \ref{fig:Scr}} for a schematic flow of our method), a group theoretic approach based on the \emph{Schreier-Coset Graph} constructed from the cosets of the special linear group \(SL(2,\mathbb{Z}_n)\).  

\paragraph{Contributions.} Our main contributions are summarized as follows:

\begin{itemize}
    \item \textbf{Schreier-Coset (SCGR) formalization :} We define the novel construction of the Schreier-Coset graphs and their application in rewiring augmentations in GNNs. The vertices corresponding to cosets of the  \(SL(2,\mathbb{Z}_n)\) modulo an upper-triangle subgroup, with constant generators yielding a \(d\)-regular graph.
    \item \textbf{Theoretical Analysis :} We provide strong theoretical  guarantees of SCGR  by defining and deriving the node-coset mappings, spectral properties, effective resistance bounds on the augmented graph, and over-squashing mitigation guarantees. 
    \item \textbf{Empirical Evaluation :} We evaluate \textbf{\textit{SCGR}} on node and graph classification benchmarks,  and test accuracy  and ER bounds by varying modularity of SBM (Stochastic Block Model) graphs.  Our extensive numerical  results demonstrate that \textbf{\textit{SCGR}} consistently matches or attains higher scores against rewiring baselines. 
\end{itemize}

\subsection{Related Works and limitations} 
Most approaches, to alleviate over-squashing,  modify the input graph by rewiring it. Expander graph-based rewiring \cite{deac2022expander}, maintains a small diameter using Cayley graphs, which can have a different node set representation than the input graph. \citet{black2023understanding} minimizes effective resistance to reduce bottlenecks. \citet{attali2024delaunay} constructs newer graphs based on node features, disregarding the original spectral properties.  FoSR \cite{karhadkar2023fosr}, improves the first order approximation of the spectral gap by adding edges.  Curvature-based methods have been implemented to enhance connectivity by adding and removing edges based on geometric principles \cite{topping2021understanding, nguyen2023revisiting}. ProxyGap \cite{jamadandi2024spectral} based on \cite{braess1968paradoxon} modifies the edges. \citet{qian2023probabilistically} explores probabilistic approaches. PANDA \citep{choi2024panda} proposes alternative message passing mechanisms based on width. \citet{finkelshtein2024cooperative} introduces a learned cooperative mechanism in message passing paradigm. These methods do improve connectivity but aggressively alter the topology, neglecting the original graph spectral properties.

\section{Schreier-Coset Rewiring for GNNs}
\subsection{Graph Preliminaries}
\textbf{Graph} Let \(G=(V,E)\) denote an undirected, connected and non-bipartite graph with node set \(V\) and edge set \(E\) and its adjacency matrix \(A \in \mathbb{R}^{n_{in} \times n_{in}}\) with entries \(A_{ij} = 1\) if \((i,j) \in E\) and \(0\) otherwise, where \(|V| = n_{in}\). The diagonal degree matrix \(D = \mathrm{diag}(d_1, \dots, d_n)\) with \(D_{vv} = d_v\). The normalized Laplacian is \(L = D^{-1/2}(D-A)D^{-1/2}\). The eigenvalues of \(L\) satisfy \(0 = \lambda_0 \leq \lambda_1 \leq \cdots \leq \lambda_{n_{in}-1}\). The eigenvector associated with \(\lambda_1\) provides canonical one-dimensional embeddings of the nodes that reflect graph connectivity.

\textbf{Special Linear Group} \(SL(2,\mathbb{Z}_n)\). Let \(\mathbb{Z}_n = \mathbb{Z}/n\mathbb{Z}\) denote the ring of integers modulo \(n\), The group \(\mathcal{G} = SL(2,\mathbb{Z}_n)\) is defined as : \[\mathcal{G} = SL(2, \mathbb{Z}_n) = \left\{ M \in \mathbb{Z}_n^{2 \times 2} \mid \det(M) \equiv 1  \pmod n \right\} ,\] where \(n\) depends on the input graph size.

\textit{Note that \(n\) is the not the same as the input graph size \(n_{in}\). We choose \(n\) as the smallest prime achieving sufficient Schreier-Coset coverage.}

\textbf{Subgroup.} Let \(H\subset SL(2,\mathbb{Z}_n)\) be a subgroup consisting of diagonal matrices with unit determinant within \(\mathcal{G}\): \[H = \left\{ \begin{pmatrix} a & 0 \\ 0 & d \end{pmatrix} \in \mathcal{G} \mid ad \equiv 1\pmod{n} \right\}.\]

\textbf{Generator.} Let \(\mathbb{S}\) be the generator set: \[\mathbb{S} = \left\{ \begin{pmatrix} 1 & \pm 1 \\ 0 & 1 \end{pmatrix}, \begin{pmatrix} 1 & 0 \\ \pm 1 & 1 \end{pmatrix} \right\} \bmod n.\]

\textbf{Expander Graph.} An expander graph is sparse yet highly connected. We use pre-computed expander graphs based on Cayley graphs \(\mathrm{Cay}(\mathcal{G}; \mathbb{S})\) derived from special linear group \(SL(2,\mathbb{Z}_n)\) with generating set \(\mathbb{S}\). While these graphs have good expansion properties, achieving large node counts is often impractical since \(|V( \mathrm{Cay}(\mathcal{G}; \mathbb{S}))| = n^3 \prod_{\text{prime} \;p|n}\left(1-\frac{1}{p^2}\right)\), which creates excessive memory requirements for large \(n\). (\(n\) satisfies the \(|V( \mathrm{Cay}(\mathcal{G}; \mathbb{S}))|\)). 
\section{Schreier-Coset Graph \(\Gamma\)}
Schreier-Coset graphs provide a permutation representation of finitely generated groups on the cosets of a subgroup of \(SL(2,\mathbb{Z}_n)\). The SC-graph plays a crucial role in our method serving as an auxiliary structure that encodes robust expansion and mixing behavior through group-theoretic symmetries. Formally, for a group \(\mathcal{G}\), a subgroup \(H \subseteq \mathcal{G}\), a generating set \(\mathbb{S} \subseteq \mathcal{G}\), the Schreier-coset graph \(\Gamma = (V_\Gamma, E_\Gamma)\) is defined as:
\begin{itemize}
    \item \textbf{Vertex set:} $V_{\Gamma} = \left \{ gH: g \in \mathcal{G} \right\}$  (collection of left cosets). 
    \item \textbf{Edge set:} For each \(gH \in V_{\Gamma}\) and each \(s \in \mathbb{S}\) includes as undirected edge \(gH,(sg)H\) in \(E_\Gamma\).
\end{itemize}
This yields a \(d\)-regular graph with \(d = |\mathbb{S}|\), since each coset has one neighbor for every generator. In constructing the the Schreier graph,  we employ a \emph{canonical construction}. That is, \(\Gamma\) is constructed over the group \(\mathcal{G} = SL(2,\mathbb Z_n)\) with subgroup \(H\) consisting of diagonal matrices, and use elementary row operations as generators: \(\mathbb{S} = \left\{ \begin{pmatrix} 1 & \pm1 \\ 0 & 1 \end{pmatrix}, \begin{pmatrix} 1 & 0 \\ \pm1 & 1 \end{pmatrix} \right\} \mod n \). The resulting Schreier-coset graph has \(|V_\Gamma| = SL(2,\mathbb Z_n)/ |H| = n(n^2-1) / \phi(n)\) vertices, where \(\phi(n)\) is Euler's totient function.
\subsection{Schreier-Guided Graph Rewiring}
We augment the input graph as seen in \textbf{Algorithm~\ref{alg:scgr}}. \(G_{in} = (V_{in},E_{in})\) using structure preserving rewiring guided by the Schreier-Coset graph \(\Gamma\). Crucial to this method is a locality-preserving map \(\phi: V_{in} \to V_{\Gamma}\). 

\paragraph{Spectral Mapping Construction.}
Let \(L_{in}, L_\Gamma\) be normalized Laplacian with eigenvectors \(\psi_i, \phi_{in}\). We define \(\Phi_{\mathrm{in}}(v)=(\psi_2(v),\ldots,\psi_{r+1}(v))\) and \(\Phi_\Gamma(x)=(\varphi_2(x),\ldots,\varphi_{r+1}(x))\in\mathbb{R}^r\)

We choose \(\phi\) by:
\(
\min_{\phi}\ \sum_{(u,v)\in E_{\mathrm{in}}}\mathrm{dist}_\Gamma(\phi(u),\phi(v))
\quad \)
Such that
\(\quad
\|\Phi_\Gamma(\phi(v))-\Phi_{\mathrm{in}}(v)\|_2\le \varepsilon\ \ \forall v.
\)
If \(V_{\mathrm{in} |>|V_\Gamma|}\) use \(q=\lceil |V_{\mathrm{in}}|/|V_\Gamma|\rceil\) disjoint copies of \(\Gamma\) or \(\Gamma \times K_q\) and apply the above per copy. 

\begin{algorithm}[t]
\caption{Schreier--Coset Transform: Graph Rewiring Method (Section 3.1)}
\label{alg:scgr}

\begin{algorithmic}[0]
\Require \textbf{Input graph} $G_{\mathrm{in}}=(V_{\mathrm{in}},E_{\mathrm{in}})$ with features $X_{\mathrm{in}}$;
coupling strength $\epsilon>0$; selection strategy $S$; Boolean \texttt{spectral\_map}
\Ensure \textbf{Augmented graph} $G^{\mathrm{rwd}}$
\Statex \quad $ = (V^{\mathrm{rwd}},E^{\mathrm{rwd}},X^{\mathrm{rwd}},w^{\mathrm{rwd}})$

\State $n \gets \textsc{FindN}(|V_{\mathrm{in}}|, S)$
\State $\Gamma \gets \textsc{SchreierTransform}(n)$
, {$\Gamma=(V_\Gamma,E_\Gamma)$}

\If{\texttt{spectral\_map}}
    \State $\Phi_{\mathrm{in}} \gets G_{\mathrm{in}}$
    \State $\Phi_{\Gamma} \gets \Gamma$
    \State $\phi \gets
    \textsc{FiedlerRanking}(V_{\mathrm{in}},V_\Gamma,
    \Phi_{\mathrm{in}},\Phi_{\Gamma})$
\Else
    \State $\phi(u) \gets
    1+((\mathrm{idx}(u)-1)\bmod |V_\Gamma|)
    \quad \forall u\in V_{\mathrm{in}}$
\EndIf

\Statex {\textbf{Augment nodes and features}}
\State $V^{\mathrm{rwd}}
\gets V_{\mathrm{in}}\cup(V_\Gamma+|V_{\mathrm{in}}|)$
\State $X_{\Gamma}
\gets \textsc{LiftFeatures}(X_{\mathrm{in}},\phi,V_\Gamma)$
\State $X^{\mathrm{rwd}}
\gets [X_{\mathrm{in}}\Vert X_{\Gamma}]$

\Statex {\textbf{Augment edges}}
\State $E_{\Gamma}^{\uparrow}
\gets
\{(a+|V_{\mathrm{in}}|,b+|V_{\mathrm{in}}|):(a,b)\in E_\Gamma\}$
\State $E_{\mathrm{cpl}}
\gets
\{(u,\phi(u)+|V_{\mathrm{in}}|),
(\phi(u)+|V_{\mathrm{in}}|,u):u\in V_{\mathrm{in}}\}$
\State $E^{\mathrm{rwd}}
\gets E_{\mathrm{in}}\cup E_{\Gamma}^{\uparrow}\cup E_{\mathrm{cpl}}$

\Statex {\textbf{Edge weights}}
\State $w^{\mathrm{rwd}}(e)\gets 1
\quad \forall e\in
E_{\mathrm{in}}\cup E_{\Gamma}^{\uparrow}$
\State $w^{\mathrm{rwd}}(e)\gets \epsilon
\quad \forall e\in E_{\mathrm{cpl}}$

\State \Return
$G^{\mathrm{rwd}}=
(V^{\mathrm{rwd}},E^{\mathrm{rwd}},
X^{\mathrm{rwd}},w^{\mathrm{rwd}})$

\end{algorithmic}
\end{algorithm}

\textit{For \textsc{FindN}, \textsc{FiedlerRanking}, \textsc{LiftFeatures}  and \textsc{SchreierTransform}in \textbf{Algorithm~\ref{alg:scgr}} pseudocodes, refer to Appendix A.8.1-A.8.3.}

\paragraph{Rewiring Strategy.} A hierarchial communication layer is introduced, and  it facilitates global information flow while preserving the sparse nature of the input graph. The rewired graph \(G^{rwd}\) is constructed as \textit{\textbf{Algorithm~\ref{alg:scgr}}}:
\begin{itemize}[noitemsep]
    \item We generate an auxiliary 4-regular Schreier-Coset graph \(\Gamma = (V_{\Gamma},E_{\Gamma})\).
    \item We establish mapping \(\phi: V_{in} \to V_{\Gamma}\) using Fiedler rankings to ensure nodes are coupled based on their global connectivity.
    \item We then define the augmented edge set \(E^{rwd}\) by concatenating the original, augmented and coupling layer.
\end{itemize}
Here, \(E_{in}\): original input edges to maintain local inductive bias), \(E_{\Gamma}^{\uparrow}\): 4-regular expander edges shifted to auxiliary nodes. \(E_{cpl}\): bidirectional edges between \(u \in V_{in}\) and its mapped coset \(\phi(u) \in V_{\Gamma}\).
The rewired graph is defined over the augmented vertex set \(V^{rwd} = V_{in} \cup V_{\Gamma}\):\(G^{rwd} = (V^{rwd}, E^{rwd}), \quad E^{rwd} = E_{in} \cup E_{\Gamma}^{\uparrow} \cup E_{cpl}\). 
The augmented edges are weighted according to their layer:
\(w_{e} = 1, \forall e \in E_{in} \cup E_{\Gamma}^{\uparrow}; \quad w_{e} = \epsilon, \forall e \in E_{cpl}\)
where \(\epsilon > 0\) is a global strength parameter regulating the influence of the global communication layer.

As formulated in \cite{arnaiz2025oversmoothing} the \textbf{\textit{over-squashing}} phenomenon has two components, topological complexity and computational bottlenecks. The \textit{SCGR} method rests on three key pillars: (i) locality preservation through spectral embeddings, (ii) effective
resistance reduction via alternative low-resistance pathways, and (iii) quantifiable over-squashing
mitigation while maintaining competitive accuracy. Now, we first analyze the computational complexity of the approach. The practical implementation of \textit{SCGR} involves several computational components, each with well-defined complexity bounds. 

\textbf{Graph Construction:} The Schreier-coset graph $\Gamma$ has $V_\Gamma = O(n)$ vertices for $\pmod n$ (and $O(n \cdot \mathrm{polylog}(n))$ in general case), with constant degree $d = |\mathbb{S}| = 4$. Its edge set therefore satisfies $|E_\Gamma| = O(|V_\Gamma|)$. Constructing $\Gamma$ via coset representatives and generator multiplications requires $O(|V_\Gamma|)$ group operations, which can be cached once and reused across multiple input graphs. 

\textbf{Mapping and Rewiring:} The spectral mapping \(\phi: V_{in} \to V_{\Gamma}\) is computed using the Fiedler vector of the normalized Laplacian for the input and Schreier graphs. This needs \(O(T \cdot |E|)\) operations via \(T\) iterations of power method. The augmentation process is strictly linear, involves couplings between \(u\) and \(\phi(u)\), avoiding any quadratic distance over node pairs. 

\textbf{Message Passing:} Each GNN layer on the rewired graph requires \(O(|E^{rwd}|)\) operations. Since \(E^{rwd} = |E_{in}| + |E_{\Gamma}| + |E_{cpl}|\), and the Schreier graph \(\Gamma\) is a d-regular expander ($d=4$), the auxiliary edge count is \(2|V_{\Gamma}|\). The total edge complexity is \(O(|E_{in}|+|V_{in}|)\), ensuring per-layer computational overhead remains near-linear in the input size.

\textbf{Space Complexity:} The node set is \(V^{rwd} = V_{in} \cup V_{\Gamma}\) where \(|V_{\Gamma}| \approx |V_{in}|\) is chosen dynamically to ensure coverage. The total number of edges added \((|E_{in}| + |E_{cpl})\) scales linearly as \(O(|V_{in}|)\) due to the constant degree of the Schreier generators \((d=4)\) and the one to one  nature of the spectral coupling. This leads to memory overhead which is sub-quadratic.

\section{Theoretical Properties of SCGR}

We first show that the Schreier-coset graphs $\Gamma_p$ for odd primes form an expander family with strong spectral and mixing properties; enabling efficient information propagation through low ER paths. 

\begin{lemma}[Uniform Spectral Gap]\label{lem:spectral-gap}
    For primes $p \geq 3$, let $\mathcal{G}_p=\operatorname{SL}(2,\mathbb{F}_p)$ and let $\mathbb{S}_p\subseteq \mathcal{G}_p$ be the generating set $\mathbb{S}_p = \left\{ \begin{pmatrix} 1 & \pm1 \\ 0 & 1 \end{pmatrix}, \begin{pmatrix} 1 & 0 \\ \pm1 & 1 \end{pmatrix}\right\}$, and $H_p \leq \mathcal{G}_p$ be the diagonal determinant-one subgroup. Then there exists $\gamma_0>0$ independent of $p$ such that the Schreier-coset graphs $\Gamma_p$ on the left cosets $\mathcal{G}_p/H_p$ satisfy $$1-\max_{\lambda \neq 1} |\lambda(P_{\Gamma_p})|\geq \gamma_0,$$
    where $P_{\Gamma_p}$ is the transition matrix of the random walk on $\Gamma_p$ and the maximum ranges over its eigenvalues $\lambda \neq 1$.
\end{lemma}

\begin{lemma}[Expander Mixing]\label{lem:expander-mix}
For primes $p\geq 3$ and all $t \geq 0$, the random walk matrix $P_{\Gamma_p}$ on a Schreier-coset graph $\Gamma_p$ with uniform spectral gap $\gamma_0$ satisfies
\[
\left|(P_{\Gamma_p}^t)_{iv} - \frac{1}{|V_{\Gamma_p}|}\right| \leq \left(1-\gamma_0\right)^t.
\]
If $t \geq \frac{\log(2|V_{\Gamma_p}|)}{\gamma_0}$, then $(P_{\Gamma_p}^t)_{iv} \geq \frac{1}{2|V_{\Gamma_p}|}$.
\end{lemma}

\begin{lemma}[Effective Resistance Bound]\label{lem:eff-res}
For any prime $p \geq 3$ and vertices $u,v \in V_\Gamma$,
\[
R_{\mathrm{eff}}^{\Gamma_p}(u,v) \leq \frac{2}{d\gamma_0},
\]
where $d = |\mathbb{S}|$ is the degree and $\gamma_0$ is the uniform spectral gap.
\end{lemma}

By bounding the effective resistance as a function of the uniform spectral gap $\gamma_0$, Lemma~\ref{lem:eff-res} guarantees the Schreier-coset family avoids structural bottlenecks, forming a strong communication backbone for the input graph.

\textbf{Locality Preservation via Spectral Embeddings.} The following result gives sufficient conditions under which the node--coset map preserves locality up to a controlled additive and multiplicative distortion.

\begin{theorem}[Locality Control]\label{thm:lipschitz}
Assume there exist constants $c_{\mathrm{in}}<\infty$ and $c_\Gamma>0$ such that $\Phi_{\mathrm{in}}$ and $\Phi_{\Gamma}$ satisfy
$$\|\Phi_{\mathrm{in}}(u)-\Phi_{\mathrm{in}}(v)\|_2 \leq c_{\mathrm{in}}\operatorname{dist}_{\mathrm{in}}(u,v)$$
for every $u,v \in V_{\mathrm{in}}$ and
$$\|\Phi_\Gamma(x)-\Phi_\Gamma(y)\|_2 \geq c_\Gamma \operatorname{dist}_{\Gamma}(x,y)$$
for every $x,y \in \phi(V_{\mathrm{in}})$. If $\|\Phi_\Gamma(\phi(v))-\Phi_{\mathrm{in}}(v)\|_2 \leq \epsilon$ for every $v$, then
$$\operatorname{dist}_\Gamma(\phi(u),\phi(v))\leq \frac{c_{\mathrm{in}}}{c_\Gamma}\operatorname{dist}_{\mathrm{in}}(u,v)+\frac{2\epsilon}{c_\Gamma}.$$
\end{theorem}

While the metric assumptions in Theorem~\ref{thm:lipschitz} can be restrictive for arbitrary graphs, they provide a sufficient condition for the node-coset map to preserve locality. This motivates the \texttt{spectral\_map} strategy. Laplacian eigenmaps minimize Dirichlet energy and hence favor embeddings that vary smoothly across adjacent vertices. On the other hand, \texttt{FiedlerRanking} aligns the vertices according to their positions along a spectral coordinate, yielding an approximate preservation of locality structure. However, neither the upper and lower Lipschitz bounds nor the uniform alignment condition is guaranteed. The theorem should thus be viewed as a conditional guarantee, while Fiedler rank matching serves as a practical heuristic.

\paragraph{Effective Resistance Analysis of the Rewired Graph.}
Theorem~\ref{thm:eff-rwd} shows that the coupling adds a low-resistance bypass via $\Gamma$, reducing $R_{\mathrm{eff}}$ for distant pairs and improving long-range information flow.

\begin{theorem}[Effective Resistance in Rewired Graph]\label{thm:eff-rwd}
In $G^{\mathrm{rwd}}$, for any $u,v\in V_{\mathrm{in}}$,
\[
R_{\mathrm{eff}}^{\mathrm{rwd}}(u,v)\le
\min\Big\{R_{\mathrm{eff}}^{\mathrm{in}}(u,v),\;
R_{\mathrm{eff}}^\Gamma(\phi(u),\phi(v))+\frac{2}{\varepsilon}\Big\}.
\]
\end{theorem}

Viewing the augmented graph as an electrical network, Thomson's principle ensures that the effective resistance cannot exceed the energy of a flow routed strictly through the original graph, giving $R_{\mathrm{eff}}^{\mathrm{in}}$, or routed strictly through the $\Gamma$-layer, giving the right-hand term.

\paragraph{Information Flow and Over-Squashing Mitigation.} The connection between effective resistance and information propagation in neural networks is well established. In message-passing networks, the gradient flow between distant nodes is inversely proportional to their effective resistance.

\begin{theorem}[Over-Squashing Mitigation]\label{thm:squashing}
For any $u,v\in V_{\mathrm{in}}$ with $u \neq v$,
\[
\rho(u,v):=\frac{R_{\mathrm{eff}}^{\mathrm{in}}(u,v)}{R_{\mathrm{eff}}^{\mathrm{rwd}}(u,v)}
\ge
\max\left\{1,\;\frac{R_{\mathrm{eff}}^{\mathrm{in}}(u,v)}{R_{\mathrm{eff}}^\Gamma(\phi(u),\phi(v))+\frac{2}{\varepsilon}}\right\}.
\]
\end{theorem}

If $R_{\mathrm{eff}}^{\mathrm{in}}(u,v)$ grows rapidly with distance while
$R_{\mathrm{eff}}^\Gamma(\phi(u),\phi(v))\le \frac{2}{d\gamma}$ is bounded, then
$\rho(u,v)\gtrsim {R_{\mathrm{eff}}^{\mathrm{in}}(u,v)}/({\frac{2}{d\gamma}+\frac{2}{\varepsilon}})$
can be very large.

\paragraph{Performance Guarantees.} While the previous theorems utilized single-layer information flow to establish bounds on effective resistance, optimally splitting the gradient flow across both layers in parallel yields a tighter bound on the total resistance:

\begin{theorem}[Performance Guarantees]\label{thm:perf}
    Suppose the transition matrices of random walks on the Schreier-coset graphs satisfy
    $$1-\lambda_2(P_{\Gamma_p})\geq \gamma_0>0$$
    uniformly in $p$ and define $B_0=\frac{2}{d\gamma_0}+\frac{2}{\epsilon}$. Then for every $u,v \in V_{\mathrm{in}}$ with $u \neq v$,
    $$R_{\mathrm{eff}}^{\mathrm{rwd}}(u,v) \leq \frac{R_{\mathrm{eff}}^{\mathrm{in}}(u,v)B_0}{R_{\mathrm{eff}}^{\mathrm{in}}(u,v)+B_0}.$$
    Consequently,
    $$\rho(u,v) \geq 1 +\frac{R_{\mathrm{eff}}^{\mathrm{in}}(u,v)}{B_0}.$$
\end{theorem}

These results make the impact of SCGR \emph{dimension-free}: the post-wiring resistance between any two nodes is capped by the expander parameters $d$ and $\gamma_0$ and the coupling $\varepsilon$, bypassing the input graph's topology. By explicitly bounding the improvement factor $\rho(u,v)$, we quantify how SCGR collapses the long-range bottlenecks responsible for over-squashing. Crucially, because the $\Gamma$-overlay is constant-degree, these long-range communication guarantees are achieved with \emph{near-linear} computational overhead.

\noindent\textit{All proofs are provided in the Appendix.}

\section{Experiments}

The efficacy of \textit{SCGR} is validated extensively on diverse node and graph classification benchmarks. We also conduct experiments on stochastic block models with controllable modularity to demonstrating \textit{SCGR}'s behavior across different community structures. The tests are run on Nvidia A100 GPU.

\subsection{Node Classification}
We predict the label of individual nodes given a graph, node features and a subset of labeled nodes. The task: labels are available only for a portion of the nodes, and the model must leverage both local features and graph structures to infer the labels of remaining nodes.  
For Node Classification, we use:
\textit{Amazon Computers \& Photo},
\textit{CoAuthor CS} \citep{shchur2018pitfalls},\textit{CiteSeer} ,\textit{Cora \& PubMed} \citep{sen2008collective} datasets. 

Each model is trained with \textit {epochs: 200, layers: 4,hidden dimension: 256 and dropout rate: 0.5} following the hyperparameters (refer \textbf{\textit{Appendix. A.8.5}}, \cite{DBLP:journals/corr/abs-1909-11793}). All experiments are repeated 20 times to ensure statistical robustness. Comparisons are made against standard baseline models including \textit{LogReg} \citep{chapelle2009semi}, \textit{MLP} \citep{werbos1974beyond}, \textit{GAT} \citep{velivckovic2017graph}, \textit{GCN} \citep{kipf2016semi}, \textit{MoNET} \citep{DBLP:journals/corr/abs-1909-11793}, \textit{LabelProp} \citep{article}, \textit{GraphSage(GS)}-variants \citep{DBLP:journals/corr/HamiltonYL17}. Given that the benchmark datasets exhibit balanced class distributions, test accuracy is adopted as the primary evaluation metric, reported in \textit{\textbf{Table \ref{tab:node_classification_combined}}}.

\begin{table*}
  \centering
  \footnotesize                      
  \caption{Performance comparison of SCGR against baseline models across six standard benchmark datasets (Graph-rewiring methods are omitted, as they report only graph-classification results on these benchmarks.)}
  \label{tab:node_classification_combined}
  \resizebox{\linewidth}{!}{
    \begin{tabular}{lcccccc}
      \toprule
      Model       & Am.\,Comp.             & Am.\,Photo             & CiteS.                 & Co.\,CS                & Cora                   & PubMed                 \\
      \midrule
      LogReg      & \(0.6410 \pm 0.0570\)  & \(0.7300 \pm 0.0650\)  & -                      & \(0.8640 \pm 0.0900\)  & -                      & -                      \\
      MLP         & \(0.4490 \pm 0.0580\)  & \(0.6960 \pm 0.0380\)  & \(0.5880 \pm 0.0220\)  & \(0.8830 \pm 0.0070\)  & \(0.5980 \pm 0.0240\)  & \(0.7010 \pm 0.0070\)  \\
      GAT         & \(0.7800 \pm 0.1900\)  & \(0.8570 \pm 0.2030\)  & \(0.6890 \pm 0.0170\) & \(0.9050 \pm 0.0060\)  & \(\mathbf{0.8080 \pm 0.0160}\) & \(0.7780 \pm 0.0210\)  \\
      GCN         & \(0.8260 \pm 0.0240\)  & \(0.9120 \pm 0.0120\) & \(0.6820 \pm 0.0160\)  & \(0.9111 \pm 0.0050\)  & \(0.7910 \pm 0.0180\)  & \(0.7880 \pm 0.0060\)  \\
      MoNET  & \(0.8350 \pm 0.0220\) & \(0.9120 \pm 0.0130\) & \(\mathbf{0.7120 \pm 0.0020}\) & \(0.9080 \pm 0.0600\) & \(0.5980 \pm 0.0080\) & \(0.7860 \pm 0.0230\)\\
      LabelProp & \(0.7080 \pm 0.0810\) & \(0.7260 \pm 0.0111\) & \(0.6780 \pm 0.0210\) & \(0.7360 \pm 0.0390\) & \(0.5050 \pm 0.0150\) & \(0.7050 \pm 0.0530\)\\
      LabelProp NL & \(0.7500 \pm 0.0390\) & \(0.8390 \pm 0.0270\) & \(0.6670 \pm 0.0220\) & \(0.7600 \pm 0.0140\) & \(0.5100 \pm 0.0100\) & \(0.7230 \pm 0.0290\)\\
      GS-mean & \(0.8240 \pm 0.0180\) & \(0.9140 \pm 0.0130\) & \(0.7160 \pm 0.0190\) & \(0.9130 \pm 0.0280\) & \(0.5860 \pm 0.0160\) & \(0.7740 \pm 0.0220\) \\
      GS-maxpool & \(-\) & \(0.9040 \pm 0.0130\) & \(0.6750 \pm 0.0230\) & \(0.8500 \pm 0.0110\) & \(0.4700 \pm 0.0150\) & \(0.7610 \pm 0.0230\) \\
      GS-meanpool & \(0.8960 \pm 0.0090\) & \(0.9070 \pm 0.0160\) & \(0.6860 \pm 0.0240\) & \(0.8960\pm 0.00090\) & \(0.4050 \pm 0.0150\) & \(0.7650 \pm 0.0240\)\\
      \midrule
 \textbf{SCGR + GCN}         & \(\mathbf{0.9031 \pm 0.0062}\) & \(\mathbf{0.9400 \pm 0.0026}\)  & \( 0.6180 \pm 0.0215 \)  & \(\mathbf{0.9211 \pm 0.0022}\) & \(0.7957 \pm 0.0058\)  & \(\mathbf{0.7894 \pm 0.0097}\) \\
      \bottomrule
    \end{tabular}}
\end{table*}

\emph{SCGR} achieves the highest accuracy on four of the six benchmark datasets, with notable improvements in \textit{Amazon Computers, Amazon Photo, Coauthor-CS }, \textit{PubMed} and is competent in \textit{Cora}. The consistent performance gains across most datasets, combined with reduced variance suggest that \emph{SCGR} provides a robust enhancement to existing GNN architectures. The method's effectiveness is particularly pronounced on the Amazon datasets and Computer Science, where the spectral properties and community structure align well with the Schreier-coset graph's expander properties, enabling more effective long-range information propagation during message passing. 
SCGR's weaker performance on \textit{CiteSeer} stems from the dataset's properties: high inter-to-intra-class edge ratio, sparse noisy features, and many isolated nodes, which make Fiedler-vector alignment unreliable and introduce misdirected coupling edges. This reflects a general limitation of spectral alignment rather than SCGR specifically.


\subsection{Graph Classification}
In Graph classification, we predict a single label for an entire graph by leveraging its structural information and associated node or edge features.
We use the \textbf{\textit{TU Dataset}} \citep{morris2020tudataset}, which comprises of over 120 graph classification and regression datasets. Representative datasets included in our experimentation are:  chemical graphs (MUTAG), protein structures (PROTEINS), social networks (IMDB-BINARY, REDDIT-BINARY), and research collaboration graphs (COLLAB). The topology of the graphs about the task is identified as requiring long-range interactions. \textit{SCGR} is compared against \textit{CGP} \citep{wilson2024cayley}, \textit{EGP} \citep{deac2022expander}, \textit{FA} \citep{alon1984eigenvalues}, \textit{DIGL} \citep{gasteiger2019diffusion}, \textit{SDRF} \citep{topping2021understanding}, \textit{FoSR} \citep{karhadkar2023fosr}, \textit{BORF} \citep{nguyen2023revisiting} and \textit{GTR} \citep{black2023understanding}.

Each model is trained with a  train/val/test split of 80\(\%\) /10 \(\%\)/10 \(\%\) and the following parameters \textit{epochs: 100, layers: 5, hidden dimension: 128, dropout: 0.5, weight decay: 0.0005}. Accuracy is the primary metric comparison.

Each experiment is run 20 times for statistical robustness.

\begin{table*}[t]
  \centering
  \caption{Results of SCGR compared against multiple models.  OOT indicates out-of-time and OOM points to out-of-memory error.
 The colors highlight \textcolor{red}{First}, \textcolor{blue}{Second} and \textcolor{gray}{Third} positions respectively.}
  \label{tab:tu_simplified_results}
  \resizebox{\textwidth}{!}{%
    \begin{tabular}{lcccccc}
      \toprule
      Model & REDDIT-BINARY & IMDB-BINARY & MUTAG & ENZYMES & PROTEINS & COLLAB \\
      \midrule
      GCN & \(77.735 \pm 1.586\) & \(60.500 \pm 2.729\) & \(74.750 \pm 4.030\) & \textcolor{blue}{\(29.083 \pm 2.363\)} & \(66.652 \pm 1.933\) & \(70.490 \pm 1.628\) \\
      + FA & OOM & \(48.950 \pm 1.652\) & \(70.250 \pm 4.608\) & \(28.667 \pm 3.693\) & \(71.071 \pm 1.506\) & \(72.039 \pm 0.771\) \\
      + DIGL & \(77.350 \pm 1.206\) & \(49.600 \pm 2.435\) & \(70.500 \pm 5.045\) & \(30.833 \pm 1.537\) & \textcolor{gray}{\(72.723 \pm 1.420\)} & \(56.470 \pm 0.865\) \\
      + SDRF & \(77.975 \pm 1.479\) & \(59.000 \pm 2.254\) & \(74.000 \pm 3.462\) & \(26.667 \pm 2.000\) & \(67.277 \pm 2.170\) & \(71.330 \pm 0.807\) \\
      + FoSR & \(77.750 \pm 1.385\) & \(59.750 \pm 2.357\) & \(75.250 \pm 5.722\) & \(24.167 \pm 3.005\) & \(70.848 \pm 1.618\) & \(67.220 \pm 1.367\) \\
      + BORF & OOT & \(48.900 \pm 0.900\) & \(76.750 \pm 0.037\) & \(27.833 \pm 0.029\) & \(67.411 \pm 0.016\) & OOT \\
      + GTR & \textcolor{gray}{\(79.025 \pm 1.248\)} & \textcolor{gray}{\(60.700 \pm 2.079\)} & \(76.500 \pm 4.189\) & \(25.333 \pm 2.931\) & \textcolor{blue}{\(72.991 \pm 1.956\)} & \textcolor{gray}{\(72.600 \pm 1.025\)} \\
      + PANDA & \textcolor{blue}{\(87.275 \pm 1.033\)} & \textcolor{red}{\(68.350 \pm 2.346\)} & \textcolor{gray}{\(76.750 \pm 5.531\)} & \(30.667 \pm 2.019\) & \(70.134 \pm 1.518\) & \textcolor{blue}{\(73.850 \pm 0.695\)} \\
      + EGP & \(67.550 \pm 1.200\) & \(59.700 \pm 2.371\) & \(70.500 \pm 4.738\) & \(27.583 \pm 3.262\) & \textcolor{red}{\(73.304 \pm 2.516\)} & \(69.470 \pm 0.970\) \\
      + CGP & \(67.050 \pm 1.483\) & \(56.200 \pm 1.825\) & \textcolor{red}{\(83.750 \pm 3.597\)} & \textcolor{gray}{\(31.000 \pm 2.397\)} & \(73.036 \pm 1.291\) & \(69.630 \pm 0.730\) \\
      \midrule
      + \textbf{SCGR} & \textcolor{red}{\(88.430 \pm 2.0600\)} & \textcolor{blue}{\(61.600 \pm 4.870\)} & \textcolor{blue}{\(77.890 \pm 0.0921 \)} & \textcolor{red}{\(53.830 \pm 0.0791\)} & \(72.590 \pm 4.330\) & \textcolor{red}{\(77.600\pm 4.350\)} \\
      \midrule
      GIN & \textcolor{gray}{\(84.600 \pm 1.454\)} & \textcolor{gray}{\(71.250 \pm 1.509\)} & \(80.500 \pm 5.143\) & \(35.667 \pm 2.803\) & \(70.312 \pm 1.749\) & \(71.490 \pm 0.746\) \\
      + FA & OOM & \(69.900 \pm 2.332\) & \(80.250 \pm 5.314\) & \textcolor{gray}{\(47.833 \pm 2.529\)} & \(72.902 \pm 1.419\) & \(72.740 \pm 0.786\) \\
      + DIGL & \(84.575 \pm 1.265\) & \(52.650 \pm 2.150\) & \(78.500 \pm 4.189\) & \(41.500 \pm 3.063\) & \textcolor{gray}{\(72.321 \pm 1.440\)} & \(57.620 \pm 1.010\) \\
      + SDRF & \(84.550 \pm 1.396\) & \(69.550 \pm 2.381\) & \(80.500 \pm 4.177\) & \(37.167 \pm 2.709\) & \(69.509 \pm 1.709\) & \(72.958 \pm 0.419\) \\
      + FoSR & \(85.750 \pm 1.099\) & \(69.250 \pm 1.810\) & \(80.500 \pm 4.738\) & \(28.083 \pm 2.301\) & \(71.518 \pm 1.767\) & \(71.720 \pm 0.892\) \\
      + BORF & OOT & \(70.700 \pm 0.018\) & \(79.250 \pm 0.038\) & \(34.167 \pm 0.029\) & \(70.625 \pm 0.017\) & OOT \\
      + GTR & \(85.474 \pm 0.826\) & \(69.550 \pm 1.473\) & \(79.000 \pm 3.847\) & \(31.750 \pm 2.466\) & \(72.054 \pm 1.510\) & \(71.849 \pm 0.710\) \\
      + PANDA & \textcolor{red}{\(90.325 \pm 0.867\)} & \(68.350 \pm 2.346\) & \textcolor{blue}{\(83.250 \pm 3.262\)} & \(42.167 \pm 2.286\) & \textcolor{blue}{\(72.321 \pm 1.786\)} & \textcolor{gray}{\(73.320 \pm 0.814\)} \\
      + EGP & \(77.875 \pm 1.563\) & \(68.250 \pm 1.121\) & \(81.500 \pm 4.696\) & \(40.667 \pm 3.095\) & \(70.848 \pm 1.568\) & \textcolor{blue}{\(72.330 \pm 0.954\)} \\
      + CGP & \(78.225 \pm 1.268\) & \textcolor{blue}{\(71.650 \pm 1.532\)} & \textcolor{red}{\(85.250 \pm 3.200\)} & \textcolor{blue}{\(50.083 \pm 2.242\)} & \(73.080 \pm 1.396\) & \textcolor{red}{\(73.350 \pm 0.788\)} \\
      \midrule
      + \textbf{SCGR} & \textcolor{blue}{\(86.200 \pm 2.780\)} & \textcolor{red}{\(71.700 \pm 4.450\)} & \textcolor{gray}{\(82.110 \pm 5.370\)} & \textcolor{red}{\(58.300 \pm 6.9700\)} & \textcolor{red}{\(74.290 \pm 3.8600\)} & \(67.880 \pm 2.4100\) \\ 
      \bottomrule
    \end{tabular}
  }
\end{table*}

In  \textit{\textbf{Table \ref{tab:tu_simplified_results}}} \textit{SCGR} consistently achieves strong performance across the \emph{TU Dataset} in both \emph{GCN + SCGR} and \emph{GIN + SCGR} configurations. Schreier-coset attains first place in six of the twelve configurations and shows competitive scores in ten datasets by being in the top 3. Notably in \textit{Enzymes} dataset, it  outperforms all models  in \textit{GCN + SCGR} and attains a significant gain of  \(\textit{8\%}\) in \textit{GIN + SCGR}. On \textit{REDDIT}, \textit{SCGR} attains first and second place respectively (vice-versa for IMDB). \textit{SCGR's} strong performance across diverse graphical datasets underscores its diverse applications.

\begin{figure}[ht] 
    \centering
    \includegraphics[width=0.75\linewidth]{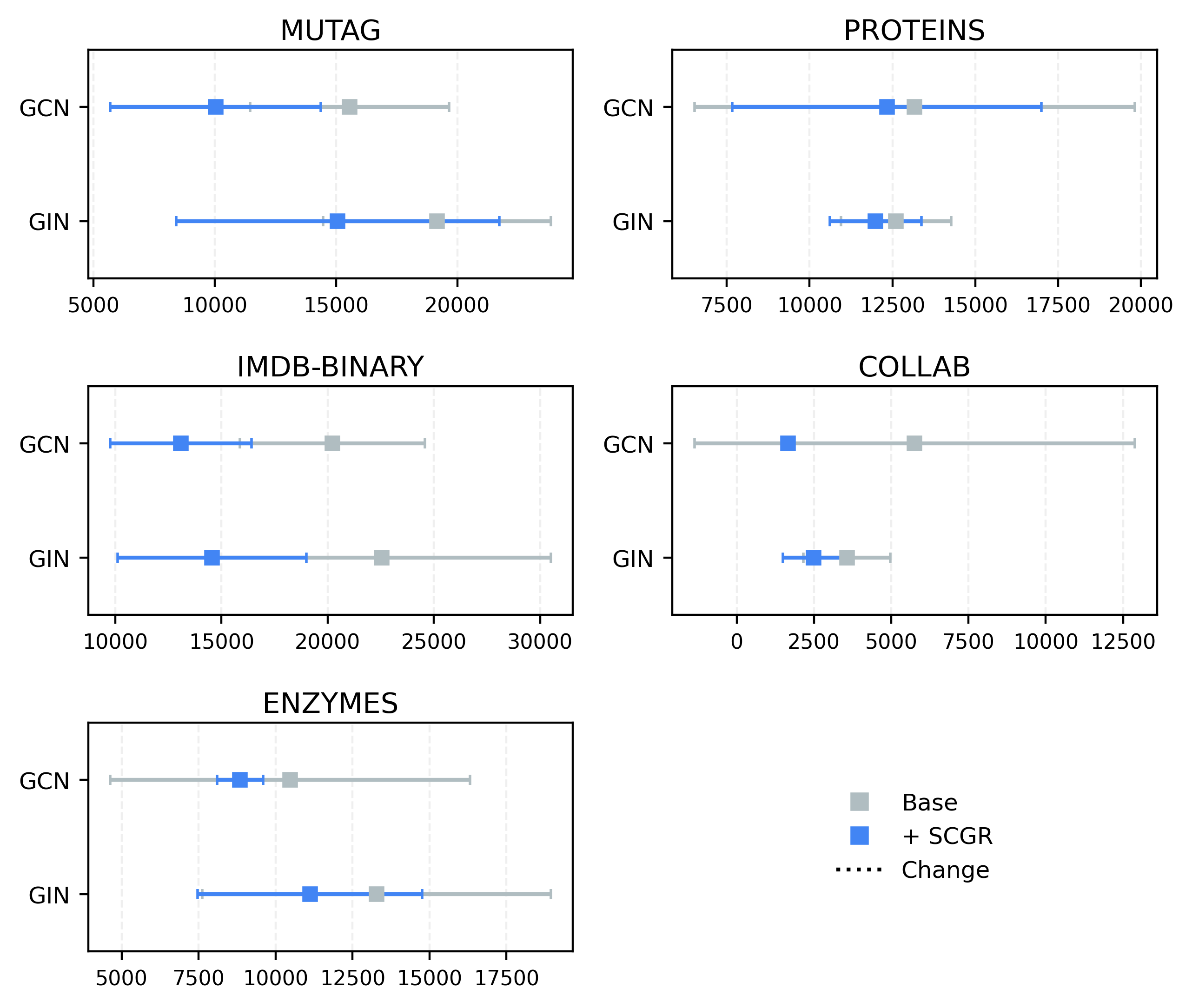}
    \caption{Effective Resistance (ER) across benchmarks: \textit{\textcolor{googleblue}{GCN/GIN + SCGR}} compared to \textcolor{gray}{GCN/GIN} baselines across varying graph complexities.}
    \label{fig:effective_resistance}
\end{figure}

\textbf{\textit{Figure \ref{fig:effective_resistance}}}, empirically validates consistent reductions in effective resistance across all benchmark datasets.
\textit{SCGR} achieves the most substantial improvements on IMDB-BINARY : \(41\%\) reduction, COLLAB: \(23-30\%\) reduction and MUTAG: \(34\%\) reduction, where long-range dependencies are particularly critical. Even on datasets with inherently good connectivity like PROTEINS, \textit{SCGR} still provides meaningful improvements. These results confirm that \textit{SCGR} successfully creates more efficient information propagation pathways, directly addressing the over-squashing. \textit{For ER values, \textbf{refer: \textit{Appendix. 8.5.2}}}

For, molecular prediction task, \textit{SCGR} is assessed on the \textbf{\textit{Open Graph Benchmark - Graph Prediction}} \textit{Ogbg-Molhiv} and \textit{Ogbg-Molpcba} datasets \cite{hu2020open}. The experimental protocol adheres to the implementation and hyperparameter configuration specified by \cite{hu2020open}.


 \textit{\textbf{Table 3, in Appendix. 8.5.4}} reports \textit{ROC-AUC\% }metrics on the Ogbg-Molhiv  and  \textit{Average Precision (AP)} Ogbg-Molpcba dataset. \textit{SCGR} exhibits robust predictive performance while maintaining high structural fidelity. Schreier-coset in both configurations attains highest ROC-AUC score in Molhiv dataset. For the Molpcba dataset, \textit{GCN+SCGR} attains the highest average precision, with \textit{GIN+SCGR} remaining competitive based on the inherent scale and structural complexity of Molpcba. 

 We further evaluate SCGR on the \textbf{\textit{Peptides-Func}} and \textbf{\textit{Peptides-Struct}} datasets from the Long Range Graph Benchmark, using test AP and MAE, respectively. SCGR achieves the best performance across both GCN and GIN, with GIN+SCGR improving AP by (13.4\%) and reducing MAE by (9.2\%) relative to the strongest baseline. Full results are reported in \textbf{\textit{Appendix 8.5.5}}

\subsection{Graph Modularity}

 Using the Stochastic Block Models (SBM) \cite{lee2019review} with 50 equal communities (1000 nodes). Intra- and inter-community edge probabilities ($p_{in}, p_{out}$) are varied to control modularity, with $p_{in} > p_{out}$ ensuring meaningful structure. This design enables systematic analysis of SCGR performance in weak-to-strong community regimes. The classification task involves predicting community membership, directly testing the model’s ability to capture long-range dependencies.

\begin{figure}[ht] 
    \centering
    \includegraphics[width=0.8\linewidth]{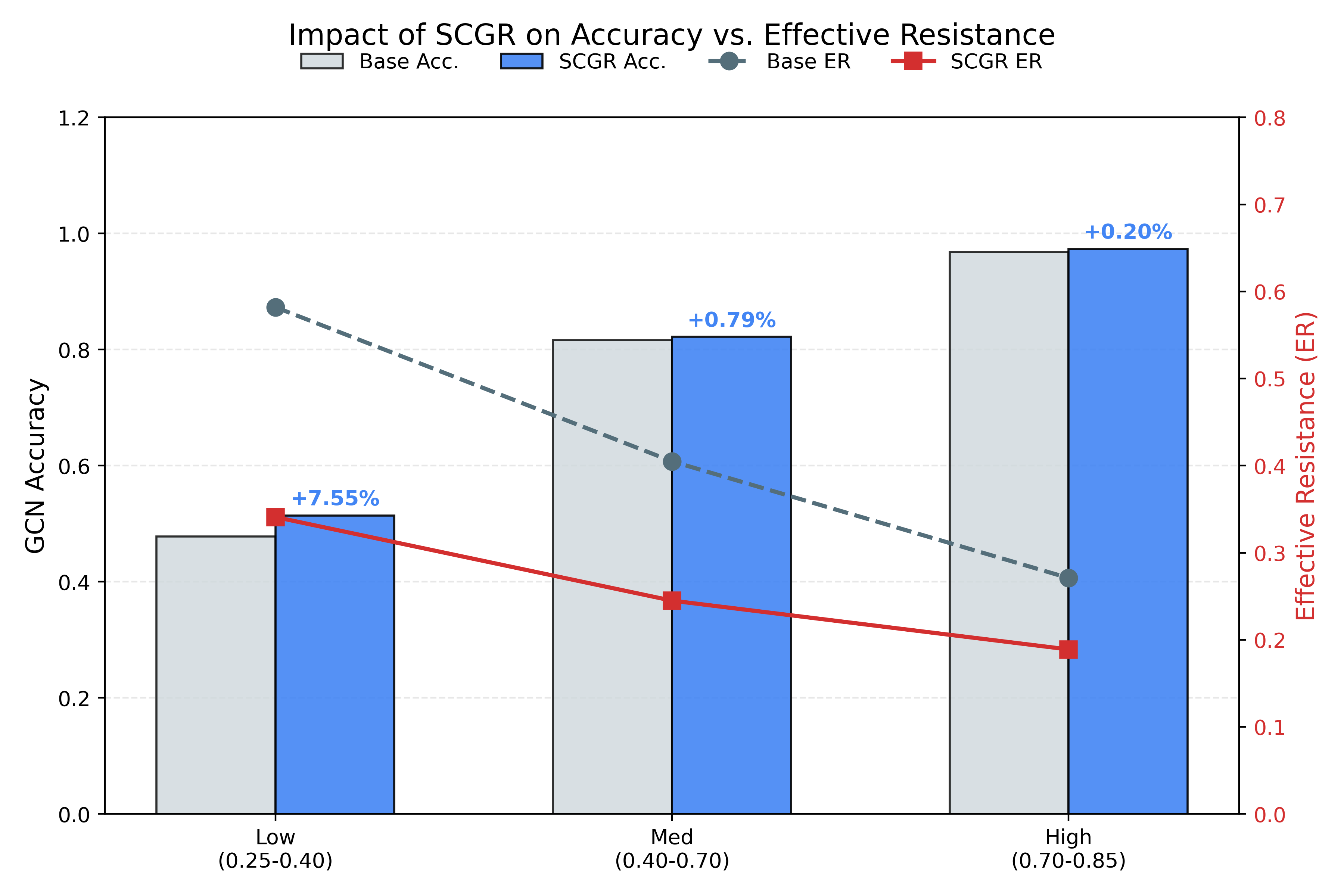}
    \caption{Comparison of effective resistance across different SBM modularities.}
    \label{fig:sbm}
\end{figure}

\textbf{\textit{Figure \ref{fig:sbm}} }reports an inverse correlation between graph modularity and \textit{SCGR} improvements. The largest gains occur in low-modularity graphs, where weak community structure induces connectivity bottlenecks with a reduction of \(41.1\%\) in effective resistance. Gains decrease with increasing modularity (+0.79\% medium, +0.20\% high), as classification becomes dominated by local neighborhood information. SCGR continues to provide a low resistance pathway. SCGR is most effective in regimes requiring long-range information propagation, where standard message passing is limited by over-squashing. (\textit{refer \textbf{\textit{Appendix. 8.5.3}} in the Appendix for ER value for varying modularity.})

\section{Ablation Study}
We validate SCGR's three design pillars through systematic ablations on graph classification (\textbf{\textit{Appendix. 8.6.1}}) and node classification (\textbf{\textit{Appendix. 8.6.2}}). \textbf{(1) Expander structure:} Replacing the $\text{SL}(2,\mathbb{Z}_n)$ Schreier-Coset graph with a random 4-regular graph degrades accuracy by $5.26\%$ on MUTAG and $2.8\%$ on Amazon Photo, confirming that algebraic symmetry provides guarantees beyond mere connectivity. \textbf{(2) Spectral alignment:} Random node-to-coset mapping under-performs Fiedler-based alignment by $6.05\%$ on MUTAG and $2.9\%$ on Amazon Photo, on COLLAB, spectral alignment achieves $71.1\%$ lower $R_{\text{eff}}$ than random mapping. \textbf{(3) Coupling strength:} Performance peaks at $\epsilon=1.0$ across both tasks---on Amazon Photo, $\epsilon=1.0$ achieves $94.0\%$ versus $92.3\%$ at $\epsilon=0.1$.

\section{Conclusion}
We introduce Schreier-Coset Graph Rewiring \textit{(SCGR)}, principled graph-rewiring framework that alleviates over-squashing in GNNs. Using the spectral properties of Schreier-Coset graph derived from \(SL(2,\mathbb{Z}_n)\) we have established an uniform upper bound on effective resistance. \(R_{\mathrm{eff}}^{\Gamma}(u,v) \leq \frac{2}{d\gamma_0}\). As a result, distant node pairs-most affected by over-squashing experience the largest reductions. 
Empirically, \textit{SCGR }reduces effective resistance by \(15-40\%\) across benchmark datasets while maintaining a competitive accuracy. In SBM graphs, where modularity varies, \textit{SCGR} gains are the strongest in lower-modularity graphs \(+7.55\%\) and remain stable in highly modular paradigms. \textit{SCGR} achieves these improvement with complexity of \(|E_{in}|+|V_{in}|\) offering a theoretically rich solution to over-squashing. From our ablation, we conclude \(\varepsilon = 1.0\) is effective to capture local and global properties, further iterations could incorporate a differentiable \(\varepsilon\) or optimization to dynamically tailor the communication layer to specific topological bottlenecks in real time. 

\bibliographystyle{plainnat}
\bibliography{reference}

@article{harary1969graph,
  title={Graph theory addison-wesley reading ma usa},
  author={Harary, F},
  journal={HARTIGAN, JA: Clustering Algorithm},
  year={1969}
}

@book{diestel2012graph,
  title={Graph theory: Springer graduate text gtm 173},
  author={Diestel, Reinhard},
  volume={173},
  year={2012},
  publisher={Reinhard Diestel}
}

@article{shuman2013emerging,
  title={The emerging field of signal processing on graphs: Extending high-dimensional data analysis to networks and other irregular domains},
  author={Shuman, David I and Narang, Sunil K and Frossard, Pascal and Ortega, Antonio and Vandergheynst, Pierre},
  journal={IEEE signal processing magazine},
  volume={30},
  number={3},
  pages={83--98},
  year={2013},
  publisher={IEEE}
}

@article{hamilton2017inductive,
  title={Inductive representation learning on large graphs},
  author={Hamilton, Will and Ying, Zhitao and Leskovec, Jure},
  journal={Advances in neural information processing systems},
  volume={30},
  year={2017}
}

@article{battaglia2018relational,
  title={Relational inductive biases, deep learning, and graph networks},
  author={Battaglia, Peter W and Hamrick, Jessica B and Bapst, Victor and Sanchez-Gonzalez, Alvaro and Zambaldi, Vinicius and Malinowski, Mateusz and Tacchetti, Andrea and Raposo, David and Santoro, Adam and Faulkner, Ryan and others},
  journal={arXiv preprint arXiv:1806.01261},
  year={2018}
}

@inproceedings{sanchez2018graph,
  title={Graph networks as learnable physics engines for inference and control},
  author={Sanchez-Gonzalez, Alvaro and Heess, Nicolas and Springenberg, Jost Tobias and Merel, Josh and Riedmiller, Martin and Hadsell, Raia and Battaglia, Peter},
  booktitle={International conference on machine learning},
  pages={4470--4479},
  year={2018},
  organization={PMLR}
}

@inproceedings{gilmer2017neural,
  title={Neural message passing for quantum chemistry},
  author={Gilmer, Justin and Schoenholz, Samuel S and Riley, Patrick F and Vinyals, Oriol and Dahl, George E},
  booktitle={International conference on machine learning},
  pages={1263--1272},
  year={2017},
  organization={Pmlr}
}

@article{koh2024physicochemical,
  title={Physicochemical graph neural network for learning protein--ligand interaction fingerprints from sequence data},
  author={Koh, Huan Yee and Nguyen, Anh TN and Pan, Shirui and May, Lauren T and Webb, Geoffrey I},
  journal={Nature Machine Intelligence},
  volume={6},
  number={6},
  pages={673--687},
  year={2024},
  publisher={Nature Publishing Group UK London}
}

@article{berg2017graph,
  title={Graph convolutional matrix completion},
  author={Berg, Rianne van den and Kipf, Thomas N and Welling, Max},
  journal={arXiv preprint arXiv:1706.02263},
  year={2017}
}

@inproceedings{jiang2019semi,
  title={Semi-supervised learning with graph learning-convolutional networks},
  author={Jiang, Bo and Zhang, Ziyan and Lin, Doudou and Tang, Jin and Luo, Bin},
  booktitle={Proceedings of the IEEE/CVF conference on computer vision and pattern recognition},
  pages={11313--11320},
  year={2019}
}

@article{kipf2016semi,
  title={Semi-supervised classification with graph convolutional networks},
  author={Kipf, TN},
  journal={ICLR,2017 arXiv:1609.02907},
  year={2017}
}

@article{velivckovic2017graph,
  title={Graph attention networks},
  author={Veli{\v{c}}kovi{\'c}, Petar and Cucurull, Guillem and Casanova, Arantxa and Romero, Adriana and Lio, Pietro and Bengio, Yoshua},
  journal={ICLR 2018, arXiv:1710.10903},
  year={2017}
}

@article{he2023message,
  title={Message passing meets graph neural networks: A new paradigm for massive MIMO systems},
  author={He, Hengtao and Yu, Xianghao and Zhang, Jun and Song, Shenghui and Letaief, Khaled B},
  journal={IEEE Transactions on Wireless Communications},
  volume={23},
  number={5},
  pages={4709--4723},
  year={2023},
  publisher={IEEE}
}

@article{alon2020bottleneck,
  title={On the bottleneck of graph neural networks and its practical implications},
  author={Alon, Uri and Yahav, Eran},
  journal={ICLR 2021, arXiv:2006.05205},
  year={2021}
}

@inproceedings{deac2022expander,
  title={Expander graph propagation},
  author={Deac, Andreea and Lackenby, Marc and Veli{\v{c}}kovi{\'c}, Petar},
  booktitle={Learning on Graphs Conference},
  pages={38--1},
  year={2022},
  organization={PMLR}
}

@article{wilson2024cayley,
  title={Cayley graph propagation},
  author={Wilson, JJ and Bechler-Speicher, Maya and Veli{\v{c}}kovi{\'c}, Petar},
  journal={In Proceedings of the Third Learning on Graphs Conference, arXiv:2410.03424},
  year={2024}
}

@inproceedings{fesser2024mitigating,
  title={Mitigating over-smoothing and over-squashing using augmentations of forman-ricci curvature},
  author={Fesser, Lukas and Weber, Melanie},
  booktitle={Learning on Graphs Conference},
  pages={19--1},
  year={2024},
  organization={PMLR}
}

@article{karhadkar2023fosr,
  title={FoSR: First-order spectral rewiring for addressing oversquashing in GNNs},
  author={Karhadkar, Kedar and Banerjee, Pradeep Kr and Mont{\'u}far, Guido},
  journal={ICLR 2023, arXiv:2210.11790},
  year={2023}
}

@inproceedings{black2023understanding,
  title={Understanding oversquashing in gnns through the lens of effective resistance},
  author={Black, Mitchell and Wan, Zhengchao and Nayyeri, Amir and Wang, Yusu},
  booktitle={International Conference on Machine Learning},
  pages={2528--2547},
  year={2023},
  organization={PMLR}
}

@inproceedings{attali2024delaunay,
  title={Delaunay graph: Addressing over-squashing and over-smoothing using delaunay triangulation},
  author={Attali, Hugo and Buscaldi, Davide and Pernelle, Nathalie},
  booktitle={Forty-first International Conference on Machine Learning},
  year={2024}
}

@inproceedings{li2018deeper,
  title={Deeper insights into graph convolutional networks for semi-supervised learning},
  author={Li, Qimai and Han, Zhichao and Wu, Xiao-Ming},
  booktitle={Proceedings of the AAAI conference on artificial intelligence},
  volume={32},
  number={1},
  year={2018}
}

@article{liang2023predicting,
  title={Predicting global label relationship matrix for graph neural networks under heterophily},
  author={Liang, Langzhang and Hu, Xiangjing and Xu, Zenglin and Song, Zixing and King, Irwin},
  journal={Advances in Neural Information Processing Systems},
  volume={36},
  pages={10909--10921},
  year={2023}
}

@article{liang2023tackling,
  title={Tackling long-tailed distribution issue in graph neural networks via normalization},
  author={Liang, Langzhang and Xu, Zenglin and Song, Zixing and King, Irwin and Qi, Yuan and Ye, Jieping},
  journal={IEEE Transactions on Knowledge and Data Engineering},
  volume={36},
  number={5},
  pages={2213--2223},
  year={2023},
  publisher={IEEE}
}

@article{arnaiz2025oversmoothing,
  title={Oversmoothing," Oversquashing", Heterophily, Long-Range, and more: Demystifying Common Beliefs in Graph Machine Learning},
  author={Arnaiz-Rodriguez, Adrian and Errica, Federico},
  journal={Neurips 2025 arXiv:2505.15547},
  year={2025}
}

@article{shchur2018pitfalls,
  title={Pitfalls of graph neural network evaluation},
  author={Shchur, Oleksandr and Mumme, Maximilian and Bojchevski, Aleksandar and G{\"u}nnemann, Stephan},
  journal={arXiv preprint arXiv:1811.05868},
  year={2018}
}

@article{sen2008collective,
  title={Collective classification in network data},
  author={Sen, Prithviraj and Namata, Galileo and Bilgic, Mustafa and Getoor, Lise and Galligher, Brian and Eliassi-Rad, Tina},
  journal={AI magazine},
  volume={29},
  number={3},
  pages={93--93},
  year={2008}
}

@article{werbos1974beyond,
  title={Beyond regression: New tools for prediction and analysis in the behavioral sciences},
  author={Werbos, Paul},
  journal={PhD thesis, Committee on Applied Mathematics, Harvard University, Cambridge, MA},
  year={1974}
}

@article{gasteiger2019diffusion,
  title={Diffusion improves graph learning},
  author={Gasteiger, Johannes and Wei{\ss}enberger, Stefan and G{\"u}nnemann, Stephan},
  journal={Advances in neural information processing systems},
  volume={32},
  year={2019}
}

@inproceedings{nguyen2023revisiting,
  title={Revisiting over-smoothing and over-squashing using ollivier-ricci curvature},
  author={Nguyen, Khang and Hieu, Nong Minh and Nguyen, Vinh Duc and Ho, Nhat and Osher, Stanley and Nguyen, Tan Minh},
  booktitle={International Conference on Machine Learning},
  pages={25956--25979},
  year={2023},
  organization={PMLR}
}

@article{chapelle2009semi,
  title={Semi-supervised learning (chapelle, o. et al., eds.; 2006)[book reviews]},
  author={Chapelle, Olivier and Scholkopf, Bernhard and Zien, Alexander},
  journal={IEEE Transactions on Neural Networks},
  volume={20},
  number={3},
  pages={542--542},
  year={2009},
  publisher={IEEE}
}

@inproceedings{alon1984eigenvalues,
  title={Eigenvalues, expanders and superconcentrators},
  author={Alon, Noga and Milman, Vitali D},
  booktitle={25th Annual Symposium onFoundations of Computer Science, 1984.},
  pages={320--322},
  year={1984},
  organization={IEEE}
}

@article{topping2021understanding,
  title={Understanding over-squashing and bottlenecks on graphs via curvature},
  author={Topping, Jake and Di Giovanni, Francesco and Chamberlain, Benjamin Paul and Dong, Xiaowen and Bronstein, Michael M},
  journal={arXiv preprint arXiv:2111.14522},
  year={2021}
}

@inproceedings{eliasof2023graph,
  title={Graph positional encoding via random feature propagation},
  author={Eliasof, Moshe and Frasca, Fabrizio and Bevilacqua, Beatrice and Treister, Eran and Chechik, Gal and Maron, Haggai},
  booktitle={International Conference on Machine Learning},
  pages={9202--9223},
  year={2023},
  organization={PMLR}
}

@article{jamadandi2024spectral,
  title={Spectral graph pruning against over-squashing and over-smoothing},
  author={Jamadandi, Adarsh and Rubio-Madrigal, Celia and Burkholz, Rebekka},
  journal={Advances in Neural Information Processing Systems},
  volume={37},
  pages={10348--10379},
  year={2024}
}

@article{braess1968paradoxon,
  title={{\"U}ber ein Paradoxon aus der Verkehrsplanung},
  author={Braess, Dietrich},
  journal={Unternehmensforschung},
  volume={12},
  number={1},
  pages={258--268},
  year={1968},
  publisher={Springer}
}

@article{qian2023probabilistically,
  title={Probabilistically rewired message-passing neural networks},
  author={Qian, Chendi and Manolache, Andrei and Ahmed, Kareem and Zeng, Zhe and Broeck, Guy Van den and Niepert, Mathias and Morris, Christopher},
  journal={In ICLR arXiv:2310.02156},
  year={2024}
}

@article{choi2024panda,
  title={Panda: Expanded width-aware message passing beyond rewiring},
  author={Choi, Jeongwhan and Park, Sumin and Wi, Hyowon and Cho, Sung-Bae and Park, Noseong},
  journal={IN ICLR arXiv:2406.03671},
  year={2024}
}

@article{finkelshtein2024cooperative,
  title={Cooperative graph neural networks},
  author={Finkelshtein, Ben and Huang, Xingyue and Bronstein, Michael and Ceylan, Ismail Ilkan},
  journal={In ICML arXiv:2310.01267},
  year={2024}
}

@article{lee2019review,
  title={A review of stochastic block models and extensions for graph clustering},
  author={Lee, Clement and Wilkinson, Darren J},
  journal={Applied Network Science},
  volume={4},
  number={1},
  pages={1--50},
  year={2019},
  publisher={Springer}
}

@article{hu2020open,
  title={Open graph benchmark: Datasets for machine learning on graphs},
  author={Hu, Weihua and Fey, Matthias and Zitnik, Marinka and Dong, Yuxiao and Ren, Hongyu and Liu, Bowen and Catasta, Michele and Leskovec, Jure},
  journal={Advances in neural information processing systems},
  volume={33},
  pages={22118--22133},
  year={2020}
}

@article{cai2023effective,
  title={Effective Resistances in Non-Expander Graphs},
  author={Cai, Dongrun and Chen, Xue and Peng, Pan},
  journal={arXiv preprint arXiv:2307.01218},
  year={2023}
}

@article{morris2020tudataset,
  title={Tudataset: A collection of benchmark datasets for learning with graphs},
  author={Morris, Christopher and Kriege, Nils M and Bause, Franka and Kersting, Kristian and Mutzel, Petra and Neumann, Marion},
  journal={arXiv preprint arXiv:2007.08663},
  year={2020}
}

@misc{rampášek2023recipegeneralpowerfulscalable,
      title={Recipe for a General, Powerful, Scalable Graph Transformer}, 
      author={Ladislav Rampášek and Mikhail Galkin and Vijay Prakash Dwivedi and Anh Tuan Luu and Guy Wolf and Dominique Beaini},
      year={2023},
      eprint={2205.12454},
      archivePrefix={arXiv},
      primaryClass={cs.LG},
      url={https://arxiv.org/abs/2205.12454}, 
}

@article{bourgain2008uniform,
  title={Uniform expansion bounds for Cayley graphs of},
  author={Bourgain, Jean and Gamburd, Alex},
  journal={Annals of Mathematics},
  pages={625--642},
  year={2008},
  publisher={JSTOR}
}

@article{hoory2006expander,
  title={Expander graphs and their applications},
  author={Hoory, Shlomo and Linial, Nathan and Wigderson, Avi},
  journal={Bulletin of the American Mathematical Society},
  volume={43},
  number={4},
  pages={439--561},
  year={2006}
}

@article{breuillard2015expansion,
  title={Expansion in finite simple groups of Lie type},
  author={Breuillard, Emmanuel and Green, Ben J and Guralnick, Robert M and Tao, Terence},
  journal={Journal of the European Mathematical Society},
  volume={17},
  number={6},
  pages={1367--1434},
  year={2015}
}

@article{DBLP:journals/corr/abs-1909-11793,
  author       = {John Palowitch and
                  Bryan Perozzi},
  title        = {{MONET:} Debiasing Graph Embeddings via the Metadata-Orthogonal Training
                  Unit},
  journal      = {CoRR},
  volume       = {abs/1909.11793},
  year         = {2019},
  url          = {http://arxiv.org/abs/1909.11793},
  eprinttype   = {arXiv},
  eprint       = {1909.11793},
  bibsource    = {dblp computer science bibliography, https://dblp.org}
}

@article{article,
author = {Zhu, Xiaojin and Ghahramani, Zoubin},
year = {2003},
month = {07},
pages = {},
title = {Learning from Labeled and Unlabeled Data with Label Propagation}
}

@article{DBLP:journals/corr/HamiltonYL17,
  author       = {William L. Hamilton and
                  Rex Ying and
                  Jure Leskovec},
  title        = {Inductive Representation Learning on Large Graphs},
  journal      = {CoRR},
  volume       = {abs/1706.02216},
  year         = {2017},
  url          = {http://arxiv.org/abs/1706.02216},
  eprinttype   = {arXiv},
  eprint       = {1706.02216},
  bibsource    = {dblp computer science bibliography, https://dblp.org}
}
\newpage

\section{Appendix.}
\subsection{Summary}
The \textbf{\emph{Appendix}} contains mathematical formulation of the Schreier-Coset Graphs. 
\begin{enumerate}
    \item \textbf{\textit{Section 8.2:}} Preliminaries and Notations.
    \item \textbf{\textit{Section 8.3:}} Contains proofs for:
    \begin{enumerate}
        \item \textbf{\textit{Lemma 4.1}} - Spectral Gap as \emph{\textbf{A.8.3, Lemma 8.1}}
        \item \textbf{\textit{Lemma 4.2}} - Expander Mixing as \emph{\textbf{A.8.3, Lemma 8.2}}
        \item \textbf{\textit{Lemma 4.3}} - Effective Resistance Bound as \emph{\textbf{A.8.3, Lemma 8.3}}
        \item \textbf{\textit{Theorem 4.4}} - Lipschitz Locality as \emph{\textbf{A.8.3, Theorem 8.4}}
        \item \textbf{\textit{Theorem 4.5}} - Effective Resistance in Rewired Graph \textbf{\(\Gamma\)} as \emph{\textbf{A.8.3, Theorem 8.5}}
        \item \textbf{\textit{Theorem 4.6}} - Over-squashing Mitigation as \emph{\textbf{A.8.3, C.8.6}}
    \end{enumerate}
    \item \textbf{\textit{Section 8.4:}} Algorithm Functions
    \begin{enumerate}
        \item \textbf{\textit{8.4.1}}- \textsc{FindN}
        \item \textbf{\textit{8.4.2}}- \textsc{FiedlerRanking}
        \item \textbf{\textit{8.4.3}}- \textsc{LiftFeatures}
        \item \textbf{\textit{8.4.4}}- \textsc{SchreierTransform}
    \end{enumerate}
    \item \textbf{\textit{Section 8.5:}} Experiments
    \begin{enumerate}
        \item \textbf{\textit{8.5.1}}- Hyper-parameters
        \item \textbf{\textit{8.5.2}}- Effective Resistance: Benchmark Datasets 
        \item \textbf{\textit{8.5.3}}- Effective Resistance: Stochastic Block Models
        \item \textbf{\textit{8.5.4}}- OGBG - Molhiv, OGBG - Molpcba
        
        \item \textbf{\textit{8.5.5}}- LRGB Peptides Struct and  Peptides Func.
    \end{enumerate}
    \item \textbf{\textit{Section 8.6:}} Ablation Studies
    \begin{enumerate}
        \item \textbf{\textit{Section 8.6.1}}- Node Classification
        \item \textbf{\textit{Section 8.6.2}}- Graph Classification
    \end{enumerate}
\end{enumerate}

\subsection{Preliminaries and Notations}
Let $\mathcal{G}$ be a finitely generated group with identity $e$, $H \subseteq \mathcal{G}$ a subgroup, and $\mathbb{S} \subseteq \mathcal{G}$ a symmetric generating set ($s \in \mathbb{S} \Rightarrow s^{-1} \in \mathbb{S}$). The quotient space $\mathcal{G}/H = \{gH : g \in \mathcal{G}\}$ consists of right cosets.
\paragraph{The Schreier-Coset Graph.}
The Schreier-coset graph $\Gamma = (V_\Gamma, E_\Gamma)$ is defined by:
\begin{itemize}
    \item $V_\Gamma = \{gH : g \in \mathcal{G}\}$
    \item $E_\Gamma = \{\{gH, (gs)H\} : gH \in V_\Gamma, s \in \mathbb{S}\}$
\end{itemize}
 $\Gamma$ is $d$-regular with $d = |\mathbb{S}|$ 
\paragraph{Spectral Properties of the Schreier-Coset Graph}
Let $A_\Gamma$ denote the adjacency matrix of $\Gamma$, $D_\Gamma = dI$ the degree matrix, $L_\Gamma = D_\Gamma - A_\Gamma$ the Laplacian, and the transition matrix $P = \frac{1}{d}A_\Gamma$.

\subsection{Theoretical Formulation and Proofs}
\subsubsection{Uniform Spectral Gap}

\begin{lemma}[Uniform Spectral Gap]\label{lem:spectral-gap}
    For primes $p \geq 3$, let $\mathcal{G}_p=\operatorname{SL}(2,\mathbb{F}_p)$ and let $\mathbb{S}_p\subseteq \mathcal{G}_p$ be the generating set $\mathbb{S}_p = \left\{ \begin{pmatrix} 1 & \pm1 \\ 0 & 1 \end{pmatrix}, \begin{pmatrix} 1 & 0 \\ \pm1 & 1 \end{pmatrix}\right\}$, and $H_p \leq \mathcal{G}_p$ be the diagonal determinant-one subgroup. Then there exists $\gamma_0>0$ independent of $p$ such that the Schreier-coset graphs $\Gamma_p$ on the left cosets $\mathcal{G}_p/H_p$ satisfy $$1-\max_{\lambda \neq 1} |\lambda(P_{\Gamma_p})|\geq \gamma_0,$$
    where $P_{\Gamma_p}$ is the transition matrix of the random walk on $\Gamma_p$ and the maximum ranges over its eigenvalues $\lambda \neq 1$.
\end{lemma}

\begin{proof}
For sufficiently large primes $p$, the Cayley graphs of $\mathcal{G}_p=\mathrm{SL}(2,\mathbb{F}_p)$ with generating set $\mathbb{S}_p$ form a family of (combinatorial) expanders \citep{bourgain2008uniform}. In other words, there exists $p_0$ such that for all $p\geq p_0$, there exists $\tilde{\gamma}>0$ independent of $p$ such that
$$\lambda_2(P_{\mathrm{Cay}(\mathcal{G}_p,\mathbb{S}_p)}) \leq 1-\tilde{\gamma}.$$
Also note that for odd primes, there are no subgroups of index-2 of $\operatorname{SL}(2,\mathbb{F}_p)$: any index-2 subgroup of $\operatorname{SL}(2,\mathbb{F}_p)$ would have to be normal and hence induce a surjective homomorphism $\varphi: \operatorname{SL}(2,\mathbb{F}_p)\to \mathbb{Z}/2\mathbb{Z}$. However, $\operatorname{SL}(2,\mathbb{F}_p)$ is generated by the matrices $g= \begin{pmatrix}
    1 & a\\ 0 & 1
\end{pmatrix}$ and 
$h=\begin{pmatrix}
    1 & 0\\ a & 1
\end{pmatrix}$ for $a \in \mathbb{F}_p$. For each $a \neq 0$, $g$ and $h$ have order $p$, and $p$ is invertible modulo 2, so
$$\varphi(I)=\varphi(g^p)=p\varphi(g)=0 \implies \varphi(g)=0,$$
and similarly for $\varphi(h)$. Hence $\varphi$ must be trivial. Thus no such index-2 subgroups can exist.

Proposition E.1 in \cite{breuillard2015expansion} states that for a finite group $G$, combinatorially expanding and no index-2 subgroups implies a two-sided expansion, i.e. there exists $\tilde{\gamma}_0 >0$ depending on the expander coefficient and number of generators $k=2$ such that
$$\max_{\lambda \neq 1} |\lambda(P_{\mathrm{Cay}(\mathcal{G}_p,\mathbb{S}_p)})| \leq 1-\tilde{\gamma}_0.$$
Since $\tilde{\gamma}$ does not depend on $p$, this bound holds uniformly over the family of Cayley graphs of $\mathcal{G}_p$ generated by $\mathbb{S}_p$ for $p \geq p_0$. 

The fact that there are no index-2 subgroups of the groups $\mathcal{G}_p$ implies that the (connected) Cayley graphs $\mathrm{Cay}(\mathcal{G}_p,\mathbb{S}_p)$ are non-bipartite (see Lemma 4 in \citep{van2021bipartite}). In particular, the eigenvalues of the transition matrices of $P_{\mathrm{Cay}(\mathcal{G}_p,\mathbb{S}_p)}$ are strictly greater than -1. Then for each $3\leq p<p_0$, there exists $\gamma_p >0$ such that $$\max_{\lambda \neq 1} |\lambda(P_{\mathrm{Cay}(\mathcal{G}_p,\mathbb{S}_p)})| \leq 1-\gamma_p.$$
Let 
$$\gamma_0=\min\left\{ \min_{3 \leq p <p_0} \gamma_p,\tilde{\gamma_0}\right\}>0.$$
Then the uniform spectral gap holds for all odd primes $p$.

Now the spectrum of $P_{\Gamma_p}$ corresponding to the Schreier-coset graph is included in the spectrum of $P_{\mathrm{Cay}(\mathcal{G}_p,\mathbb{S}_p)}$ corresponding the Cayley graph (see Proposition 11.17 and the preceding remark in \cite{hoory2006expander}), which gives us the claim.

\end{proof}

\subsubsection{Expander Mixing}
\begin{lemma}[Expander Mixing]\label{lem:expander-mix}
For primes $p\geq 3$ and all $t \geq 0$, the random walk matrix $P_{\Gamma_p}$ on a Schreier-coset graph $\Gamma_p$ with uniform spectral gap $\gamma_0$ satisfies
\[
\left|(P_{\Gamma_p}^t)_{iv} - \frac{1}{|V_{\Gamma_p}|}\right| \leq \left(1-\gamma_0\right)^t.
\]
If $t \geq \frac{\log(2|V_{\Gamma_p}|)}{\gamma_0}$, then $(P_{\Gamma_p}^t)_{iv} \geq \frac{1}{2|V_{\Gamma_p}|}$.
\end{lemma}

\begin{proof}
Fix a prime $p\geq 3$ and let $\Gamma=\Gamma_p$. Let $P = \sum_k \lambda_k u_k u_k^\top$ be the spectral decomposition with orthonormal eigenvectors $u_k$. We have $u_1 = \frac{1}{\sqrt{|V_{\Gamma}|}}\mathbf{1}$ with $\lambda_1 = 1$. Thus:
\begin{equation}
P^t = \frac{1}{|V_{\Gamma}|}\mathbf{1}\mathbf{1}^\top + \sum_{k=2}^{|V_{\Gamma}|} \lambda_k^t u_k u_k^\top
\end{equation}
Therefore
\begin{equation}
(P^t)_{iv} = \frac{1}{|V_{\Gamma}|} + \sum_{k=2}^{|V_{\Gamma}|} \lambda_k^t u_k(i)u_k(v)
\end{equation}
By Lemma~\ref{lem:spectral-gap} we have that $|\lambda_k|\leq 1-\gamma_0$ for $k\geq 2$. Hence, by Cauchy-Schwarz inequality
\begin{align}
\left|(P^t)_{iv} - \frac{1}{|V_{\Gamma}|}\right| &\leq \sum_{k=2}^{|V_{\Gamma}|} |\lambda_k|^t |u_k(i)||u_k(v)| \\
&\leq (1-\gamma_0)^t \sum_{k=2}^{|V_{\Gamma}|} |u_k(i)||u_k(v)| \\
&\leq (1-\gamma_0)^t \sqrt{\sum_{k=2}^{|V_{\Gamma}|}u_k(i)^2} \sqrt{\sum_{k=2}^{|V_{\Gamma}|} u_k(v)^2} \\
&\leq (1-\gamma_0)^t
\end{align}
where the last inequality uses $\sum_{k=1}^{|V_{\Gamma}|} u_k(i)^2 = 1$ (orthonormality).
Now note that $(1-\gamma_0)^t \leq e^{-\gamma_0t}$. Hence, if $t \geq \frac{\log(2|V_\Gamma|)}{\gamma_0}$, then
$$(1-\gamma_0)^t \leq e^{-\gamma_0 t}\leq \frac{1}{2|V_\Gamma|}.$$
It follows that $$(P^t)_{iv}\geq \frac{1}{|V_\Gamma|} -\frac{1}{2|V_\Gamma|}=\frac{1}{2|V_\Gamma|}.$$
\end{proof}

\subsubsection{Effective Resistance}
\begin{lemma}[Effective Resistance Bound]\label{lem:eff-res}
For any prime $p \geq 3$ and vertices $u,v \in V_\Gamma$,
\[
R_{\mathrm{eff}}^{\Gamma_p}(u,v) \leq \frac{2}{d\gamma_0},
\]
where $d = |\mathbb{S}|$ is the degree and $\gamma_0$ is the uniform spectral gap.
\end{lemma}

\begin{proof}
For a $d$-regular graph $\Gamma$ with transition matrix \(P\) having spectral gap \(\gamma\), Lemma 2.2, Property 3 in \cite{cai2023effective} gives us that
\begin{equation}
\frac{1}{2}\Big(\frac{1}{d(u)}+ \frac{1}{d(v)}\Big) 
\leq R_{\mathrm{eff}}(u,v) \leq \frac{1}{\lambda_2(\tilde{L})}\cdot \Big(\frac{1}{d(u)}+ \frac{1}{d(v)}\Big)
\end{equation}
where, \(\tilde{L}\) is the normalized Laplacian of \(\Gamma\). 
Therefore, we have \(d(u) = d(v) = d\) for all vertices \(u,v \in V\). Therefore:
\begin{equation}\label{eq1}
\frac{1}{d} \leq R_{\mathrm{eff}}(u,v) \leq \frac{2}{d \cdot \lambda_2(\tilde{L})}
\end{equation}
Now, for a $d$-regular graph, the normalized Laplacian is \(\tilde{L} = I-P\), where \(P\) is a transition matrix. Therefore, if \(\lambda\) is an eigenvalue of \(P\), then \(1-\lambda\) is an eigenvalue of \(\tilde{L}\). 
Specifically:
\begin{itemize}
    \item \(\lambda_1 = 1-\lambda_1(P) = 1-1=0\)
    \item \(\lambda_2(\tilde{L}) = 1-\lambda_2(P)\)
\end{itemize}
Now, applying the spectral gap condition,
\begin{equation}
\lambda_2(\tilde{L}) = 1 - \lambda_2(P) \geq 1-(1-\gamma) = \gamma
\end{equation}
Substituting this lower bound to the inequality in (\ref{eq1}), we get:
\begin{equation}
R_{\mathrm{eff}}^{\Gamma}(u,v) \leq \frac{2}{d\cdot\lambda_2(\tilde{L})} \leq \frac{2}{d \gamma}.
\end{equation}
This bound is tight up to constants, as the effective resistance can indeed approach this upper bound for pairs of vertices that are apart in the graph structure., particularly in expander graphs where the spectral gap \(\gamma\) is bounded away from zero. Now take $\gamma=\gamma_0$ as our uniform spectral gap so that this holds for all primes $p \geq 3$.
\end{proof}

\subsubsection{Spectral Mapping Construction}

Let $G_{\mathrm{in}}=(V_{\mathrm{in}},E_{\mathrm{in}})$ be the input graph. We construct a locality-preserving map $\phi:V_{\mathrm{in}}\to V_\Gamma$ (with $\phi(v)=g_vH$) as follows:\\
\textbf{Case (i)} $|V_{\mathrm{in}}|\le |V_\Gamma|$: Let $L_{\mathrm{in}},L_\Gamma$ be (normalized) Laplacians with eigenvectors $\{\psi_i\},\{\varphi_i\}$ ordered by eigenvalues. Define
$\Phi_{\mathrm{in}}(v)=(\psi_2(v),\ldots,\psi_{r+1}(v))$ and
$\Phi_\Gamma(x)=(\varphi_2(x),\ldots,\varphi_{r+1}(x))\in\mathbb{R}^r$.
Set $\phi$ by solving:
\begin{equation}
\min_{\phi:V_{\mathrm{in}}\hookrightarrow V_\Gamma}\ \sum_{(u,v)\in E_{\mathrm{in}}}\mathrm{dist}_\Gamma(\phi(u),\phi(v))
\quad \text{such that}\quad
\|\Phi_\Gamma(\phi(v))-\Phi_{\mathrm{in}}(v)\|_2\le \varepsilon\ \ \forall v .
\end{equation}
\textbf{Case (ii)} $|V_{\mathrm{in}}|>|V_\Gamma|$: Use disjoint copies $\Gamma^{(1)},\ldots,\Gamma^{(q)}$ (or $\Gamma\times K_q$) and apply (i) per block.\\

\begin{theorem}[Locality Control]\label{thm:lipschitz}
Assume there exists constants $c_{\mathrm{in}}<\infty$ and $c_\Gamma>0$ such that $\Phi_{\mathrm{in}}$ and $\Phi_{\Gamma}$ satisfy
$$\|\Phi_{\mathrm{in}}(u)-\Phi_{\mathrm{in}}(v)\|_2 \leq c_{\mathrm{in}}\operatorname{dist}_{\mathrm{in}}(u,v)$$
for every $u,v \in V_{\mathrm{in}}$ and
$$\|\Phi_\Gamma(x)-\Phi_\Gamma(y)\|_2 \geq c_\Gamma \operatorname{dist}_{\Gamma}(x,y)$$
for every $x,y \in \phi(V_{\mathrm{in}})$. If $\|\Phi_\Gamma(\phi(v))-\Phi_{\mathrm{in}}(v)\|_2 \leq \epsilon$ for every $v$, then
$$\operatorname{dist}_\Gamma(\phi(u),\phi(v))\leq \frac{c_{\mathrm{in}}}{c_\Gamma}\operatorname{dist}_{\mathrm{in}}(u,v)+\frac{2\epsilon}{c_\Gamma}.$$
\end{theorem}

\begin{proof}
    By assumption, there exists constants $c_{\mathrm{in}}<\infty$ and $c_\Gamma>0$ such that $\Phi_{\mathrm{in}}$ and $\Phi_{\Gamma}$ satisfy
    $$\|\Phi_{\mathrm{in}}(u)-\Phi_{\mathrm{in}}(v)\|_2 \leq c_{\mathrm{in}}\operatorname{dist}_{\mathrm{in}}(u,v)$$
    for every $u,v \in V_{\mathrm{in}}$ and
    $$\|\Phi_\Gamma(x)-\Phi_\Gamma(y)\|_2 \geq c_\Gamma \operatorname{dist}_{\Gamma}(x,y)$$
    for every $x,y \in \phi(V_{\mathrm{in}})$.
    Since $\|\Phi_\Gamma(\phi(v))-\Phi_{\mathrm{in}}(v)\|_2 \leq \epsilon$, we have by the triangle inequality that
    \begin{align*}
        \|\Phi_\Gamma(\phi(u))-\Phi_\Gamma(\phi(v))\|_2 &\leq \|\Phi_{\mathrm{in}}(u)-\Phi_{\mathrm{in}}(v)\|_2 + 2\epsilon\\
        &\leq c_{\mathrm{in}}\operatorname{dist}_{\mathrm{in}}(u,v)+ 2\epsilon.
    \end{align*}
    Thus
    \begin{align*}
        \operatorname{dist}_\Gamma(\phi(u),\phi(v))&\leq \frac{1}{c_\Gamma}\|\Phi_\Gamma(\phi(u))-\Phi_\Gamma(\phi(v))\|_2\\
        &\leq \frac{c_{\mathrm{in}}}{c_\Gamma}\operatorname{dist}_{\mathrm{in}}(u,v)+\frac{2\epsilon}{c_\Gamma}.
    \end{align*}
\end{proof}

\subsubsection{Effective Resistance Analysis of Rewired Graph}
\textbf{Augmented System}
The augmented system couples $G_{\mathrm{in}}$ with $\Gamma$ via edges $\{(v, \phi(v))\}$ of conductance $\varepsilon > 0$.
All edges in $E_{\mathrm{in}}$ and $E_\Gamma$ have unit conductance.
Let $P\in\{0,1\}^{|V_{\mathrm{in}}|\times |V_\Gamma|}$ be the injection/matching matrix with
$P_{v,\phi(v)}=1$ and $0$ otherwise. Then the augmented Laplacian is
\begin{equation}
L_{\mathrm{aug}} =
\begin{bmatrix}
L_{\mathrm{in}} + \varepsilon I & -\varepsilon P\\
-\varepsilon P^\top & L_\Gamma + \varepsilon P^\top P
\end{bmatrix}.
\end{equation}
\medskip
\begin{theorem}[Effective Resistance in Rewired Graph]\label{thm:eff-rwd}
In the rewired graph $G^{\mathrm{rwd}}$, the effective resistance between nodes $u,v \in V_{\mathrm{in}}$ satisfies
\[
R_{\mathrm{eff}}^{\mathrm{rwd}}(u,v) \leq
\min\Big\{R_{\mathrm{eff}}^{\mathrm{in}}(u,v),\;
R_{\mathrm{eff}}^\Gamma(\phi(u),\phi(v)) + \frac{2}{\varepsilon}\Big\}.
\]
\end{theorem}

\begin{proof}
We prove this using:
\textbf{Rayleigh--Thomson Principle}
Effective resistance equals the minimum energy of a unit $u\!\to\!v$ flow.
Consider two feasible unit flows:

\emph{Route 1:} Use the minimum-energy unit flow in $G_{\mathrm{in}}$ only.
This has energy $R_{\mathrm{eff}}^{\mathrm{in}}(u,v)$, hence
$R_{\mathrm{eff}}^{\mathrm{rwd}}(u,v)\le R_{\mathrm{eff}}^{\mathrm{in}}(u,v)$.

\emph{Route 2:} Send one unit across the coupling edge $(u,\phi(u))$,
then route through $\Gamma$ from $\phi(u)$ to $\phi(v)$ using the minimum-energy
unit flow in $\Gamma$, then send one unit across $(v,\phi(v))$.
Coupling edges have conductance $\varepsilon$ and thus resistance $1/\varepsilon$,
so each coupling edge contributes energy $1/\varepsilon$ (unit flow).
The $\Gamma$ segment contributes $R_{\mathrm{eff}}^\Gamma(\phi(u),\phi(v))$
since $\Gamma$ edges have unit conductance. Therefore this route has total energy
\[
R_{\mathrm{eff}}^\Gamma(\phi(u),\phi(v))+\frac{2}{\varepsilon}.
\]
Taking the minimum over all flows yields
\[
R_{\mathrm{eff}}^{\mathrm{rwd}}(u,v)\le
\min\Big\{R_{\mathrm{eff}}^{\mathrm{in}}(u,v),\;
R_{\mathrm{eff}}^\Gamma(\phi(u),\phi(v)) + \frac{2}{\varepsilon}\Big\}.
\]
\end{proof}

\begin{theorem}[Over-Squashing Mitigation]\label{thm:squashing}
For any $u,v\in V_{\mathrm{in}}$ with $u \neq v$,
\[
\rho(u,v):=\frac{R_{\mathrm{eff}}^{\mathrm{in}}(u,v)}{R_{\mathrm{eff}}^{\mathrm{rwd}}(u,v)}
\ge
\max\left\{1,\;\frac{R_{\mathrm{eff}}^{\mathrm{in}}(u,v)}{R_{\mathrm{eff}}^\Gamma(\phi(u),\phi(v))+\frac{2}{\varepsilon}}\right\}.
\]
\end{theorem}

\begin{proof}
Let $R=R_{\mathrm{eff}}^{\mathrm{in}}(u,v)$ and $B=R_{\mathrm{eff}}^{\Gamma}(\phi(u),\phi(v))+\frac{2}{\varepsilon}$. From Theorem~\ref{thm:eff-rwd}, we have that
$$\frac{1}{R_{\mathrm{eff}}^{\mathrm{rwd}}(u,v)} \geq \max \left\{\frac{1}{R}, \frac{1}{B} \right\}.$$
Hence, since $R\geq 0$,
\begin{align*}
    \rho(u,v)&=\frac{R}{R_{\mathrm{eff}}^{\mathrm{rwd}}(u,v)}\\
    &\geq \max\left\{ \frac{R}{R},\frac{R}{B}\right\}\\
    &= \max \left\{1,\frac{R_{\mathrm{eff}}^{\mathrm{in}}(u,v)}{R_{\mathrm{eff}}^\Gamma(\phi(u),\phi(v))+\frac{2}{\varepsilon}} \right\}.
\end{align*}

\end{proof}
\paragraph{Remark A.3.1}
For pairs with large $R_{\mathrm{eff}}^{\mathrm{in}}(u,v)$, the ratio
$R_{\mathrm{eff}}^{\mathrm{rwd}}(u,v)/R_{\mathrm{eff}}^{\mathrm{in}}(u,v)$
decays as $O(1/R_{\mathrm{eff}}^{\mathrm{in}}(u,v))$, and becomes small whenever
$R_{\mathrm{eff}}^{\mathrm{in}}(u,v)\gg \frac{2}{d\gamma}+\frac{2}{\varepsilon}$.

\subsubsection{Performance Guarantees}
\begin{theorem}[Performance Guarantees]\label{thm:perf}
    Suppose the transition matrices of random walks on the Schreier-coset graphs satisfy
    $$1-\lambda_2(P_{\Gamma_p})\geq \gamma_0>0$$
    uniformly in $p$ and define $B_0=\frac{2}{d\gamma_0}+\frac{2}{\epsilon}$. Then for every $u,v \in V_{\mathrm{in}}$ with $u \neq v$,
    $$R_{\mathrm{eff}}^{\mathrm{rwd}}(u,v) \leq \frac{R_{\mathrm{eff}}^{\mathrm{in}}(u,v)B_0}{R_{\mathrm{eff}}^{\mathrm{in}}(u,v)+B_0}.$$
    Consequently,
    $$\rho(u,v) \geq 1 +\frac{R_{\mathrm{eff}}^{\mathrm{in}}(u,v)}{B_0}.$$
\end{theorem}

\begin{proof}
    Let $R=R_{\mathrm{eff}}^{\mathrm{in}}(u,v)$ and the resistance of the Schreier-coset graph by pass with conductance $\varepsilon$ by $B=R_{\mathrm{eff}}^{\Gamma_p}(\phi(u),\phi(v))+\frac{2}{\varepsilon}$. By Lemma~\ref{lem:eff-res}, we have that $R_{\mathrm{eff}}^{\Gamma_p}\leq \frac{2}{d\gamma_0}$, and so
    $$B\leq \frac{2}{d\gamma_0}+\frac{2}{\varepsilon}=B_0.$$
    Now choose an optimal unit flow $f_{\mathrm{in}}$ from $u \to v$ supported on $G_{\mathrm{in}}$ with energy $R$ and a unit flow $f_{\Gamma}$ from $u \to v$ that travels through the coupling edges and $\Gamma_p$ with energy $B$.
    Let $a\in [0,1]$ be the fraction of the current going through the input layer and $1-a$ through the Schreier graph:
    $$f_a=a\cdot f_{\mathrm{in}} + (1-a)f_\Gamma.$$
    This is another unit flow from $u \to v$ with disjoint edge supports, so its energy is
    $$E(f_a)=a^2R +(1-a)^2B.$$
    The energy is minimized when $$aR=(1-a)B,$$
    so $$a_*=\frac{B}{R+B}, \quad 1-a_*=\frac{R}{R+B}$$
    with
    $$E(f_{a_*})=\frac{RB}{R+B}.$$
    As the function $x \mapsto \frac{Rx}{R+x}$ is increasing for $x>0$ and $B\leq B_0$, we have that 
    $$R_{\mathrm{eff}}^{\mathrm{rwd}}(u,v)\leq \frac{RB_0}{R+B_0}=\frac{R_{\mathrm{eff}}^{\mathrm{in}}(u,v)B_0}{R_{\mathrm{eff}}^{\mathrm{in}}(u,v)+B_0}.$$
    For the second claim,
    \begin{align*}
        \rho(u,v) &= \frac{R_{\mathrm{eff}}^{\mathrm{in}}(u,v)}{R_{\mathrm{eff}}^{\mathrm{rwd}}(u,v)}\geq \frac{R}{RB_0/(R+B_0)}= \frac{R+B_0}{B_0}=1+\frac{R}{B_0}=1+\frac{R_{\mathrm{eff}}^{\mathrm{in}}(u,v)}{B_0}.
    \end{align*}
    
\end{proof}

\subsection{Algorithm Functions}

The following functions are sub-functions of \textbf{Algorithm 1} 

\subsubsection{FindN}

The \(\pmod n\) for the group \(SL(2,\mathbb{Z}_n)\) is selected based on input graph size. The algorithm searches for smallest \(n_{in}\) such that the resulting Schreier-Coset size \(|V_\Gamma|\) satisfies \(|V_\Gamma| \geq n_{in}\). This ensures complete coverage of input nodes. To optimize expansion properties, a \textit{smallest prime} strategy is applied to prioritize prime moduli. The results are cached to ensure the search and subsequent graph construction occur only once per unique graph size. This helps in dynamic scaling as it encompases the entire graph nodes, expands optimally by leveraging the \(SL(2,\mathbb{Z}_n)\) properties providing high-conductivity bypasses.

\begin{algorithm}[h]
\caption{Automatic Modulo Selection (\textsc{FindN})}
\label{alg:findn_compact}
\footnotesize
\begin{algorithmic}

\Require Input node count $n_{\mathrm{in}}$, search strategy $S$
\Ensure Modulo $n$ and coset-space size $|V_\Gamma|$

\If{$(n_{\mathrm{in}},S)$ is cached}
    \State \Return cached values
\EndIf

\For{$n_{\mathrm{cand}} \gets 2$ \textbf{to}
     $\max(100,n_{\mathrm{in}}+10)$}
    \State $\mathrm{size} \gets
    \textsc{EstimateSize}(n_{\mathrm{cand}})$

    \If{$\mathrm{size} \geq n_{\mathrm{in}}$}
        \If{$S=\text{'prime'} \land
        \textsc{IsPrime}(n_{\mathrm{cand}})$}
            \State \textbf{break} with
            $(n_{\mathrm{cand}},\mathrm{size})$
        \EndIf

        \If{$S=\text{'smallest'}$}
            \State \textbf{break} with
            $(n_{\mathrm{cand}},\mathrm{size})$
        \EndIf
    \EndIf
\EndFor

\State \textbf{Fallback:} \Return the cached result for $n=5$
if no candidate is found.

\end{algorithmic}
\end{algorithm}




\subsubsection{FiedlerRanking}

To establish the spectral mapping, we compute the Fiedler ranking of the nodes in both the input graph \(G_{in}\) and the auxiliary Schreier graph \(\Gamma\). We utilize an iterative power method to approximate the eigenvector corresponding to second smallest eigenvalue of the normalized Laplacian. Starting from a random vector orthogonalized against the trivial eigenvector \(\mathbf{1}\), we repeatedly apply the symmetric normalized adjacency operator and re-orthogonalize to ensure convergence. The resulting vector is a one-dimensional embedding that reflects the global connectivity and "bottleneck" structure of the graph. Nodes are then ranked by their values in this vector to determine the pairing for the constant-degree SCGR overlay 

\begin{algorithm}[h]
\caption{Spectral Node Ranking (\textsc{FiedlerRanking})}
\label{alg:fiedler}
\footnotesize
\begin{algorithmic}
\Require Graph $G$, number of nodes $n$, iterations $T$
\Ensure Sorted node indices (ranks)

\State $L_{sym} \gets$ compute symmetric normalized adjacency matrix $D^{-1/2}AD^{-1/2}$
\State $v_1 \gets \sqrt{\text{degrees}}$ \Comment{Trivial eigenvector component}
\State $v \gets \text{Random vector } \in \mathbb{R}^n$
\State $v \gets v - \text{proj}_{v_1}(v)$ \Comment{Ensure orthogonality to trivial eigenvector}

\For{$t \gets 1$ \textbf{to} $T$}
    \State $v \gets L_{sym} v$
    \State $v \gets v - \text{proj}_{v_1}(v)$ \Comment{Re-orthogonalize}
    \State $v \gets v / \|v\|$ \Comment{Normalize}
\EndFor

\State \Return $\textsc{Argsort}(v)$
\end{algorithmic}
\end{algorithm}

\subsubsection{LiftFeatures}

The Lift Features function establishes initial hidden states for auxiliary Schrier nodes. As the mapping \(\phi\) may be many-to-one depending on the ratio \(\frac{|V_{\Gamma}|)}{(n_{in}}\), the feature of the each auxiliary node is computed using the mean node embedding of all input nodes mapped to that specific coset. The average acts as a feature pool, ensuring communication layer is initialized with representation of the original graph. 

\begin{algorithm}[H]
\caption{Feature Lifting (\textsc{LiftFeatures})}
\label{alg:liftfeatures}
\footnotesize
\begin{algorithmic}
\Require Input features $X_{\mathrm{in}} \in \mathbb{R}^{n_{in} \times d}$, Spectral mapping $\phi$, Coset count $|V_\Gamma|$
\Ensure Auxiliary features $X_\Gamma \in \mathbb{R}^{|V_\Gamma| \times d}$

\State $X_\Gamma \gets \mathbf{0}^{|V_\Gamma| \times d}$ \Comment{Initialize auxiliary feature matrix}
\State $C \gets \mathbf{0}^{|V_\Gamma| \times 1}$ \Comment{Initialize node counts for averaging}

\For{$i \gets 1$ \textbf{to} $n_{in}$}
    \State $v \gets \phi(i)$ \Comment{Find the target coset node}
    \State $X_\Gamma[v] \gets X_\Gamma[v] + X_{\mathrm{in}}[i]$ \Comment{Aggregate features}
    \State $C[v] \gets C[v] + 1$ \Comment{Increment occupancy count}
\EndFor

\State $X_\Gamma \gets X_\Gamma / \max(C, 1)$ \Comment{Normalize by averaging features}
\State \Return $X_\Gamma$
\end{algorithmic}
\end{algorithm}

\subsubsection{SchreierTransform}

The SchreierTransform function is implemented as a whole method, encompassing group modulus \(n\) such that the resulting coset is at least as large as the input graph ensuring full coverage, fiedler ranking for both \(G_{in}\) and \(\Gamma\) which aligns nodes based on their global connectivity profiles. 

The final augmented graph is constructed by concatenating the original edges with the 4-regular Schreier edges and bi-directional coupling edges that link original nodes to their spectrally aligned counterparts in the coset space. Node features in auxiliary nodes are initialized by lifting and averaging the input features \(X_{in}\) through the mapping \(\phi\). 

The construction introduces highly connected bypasses that effectively reduce the graph's effective resistance (ER), facilitating long-range communication and mitigating over-squashing in bottleneck heavy grap structures. 

\begin{algorithm}[h]
\caption{Schreier-Coset Graph Rewiring (\textsc{SCGR})}
\label{alg:scgr_compact}
\footnotesize
\begin{algorithmic}
\Require Input graph $G_{in}(V, E, X)$, coupling strength $\epsilon$, mapping boolean \texttt{spectral}
\Ensure Augmented graph $G^{aug}(V', E', X')$

\State $n \gets \textsc{FindN}(|V|, \text{strategy})$ \Comment{Select optimal group modulus}
\State $\Gamma(V_\Gamma, E_\Gamma) \gets \textsc{BuildSchreierGraph}(n)$ \Comment{Construct expander overlay}

\If{\texttt{spectral}}
    \State $\phi \gets \textsc{SpectralAlign}(G_{in}, \Gamma)$ \Comment{Pair nodes via Fiedler rankings}
\Else
    \State $\phi(u) \gets u \pmod{|V_\Gamma|} \quad \forall u \in V$ \Comment{Uniform fallback mapping}
\EndIf

\State $X_\Gamma \gets \textsc{AverageFeatures}(X, \phi)$ \Comment{Initialize auxiliary node features}
\State $E_{cpl} \gets \{(u, \phi(u) + |V|), (\phi(u) + |V|, u) \mid u \in V\}$ \Comment{Define coupling edges}
\State \Return $G_{in} \cup \Gamma \cup E_{cpl}$ with weight $\epsilon$ for $E_{cpl}$
\end{algorithmic}
\end{algorithm}

\subsection{Experiments}

\subsubsection{Hyperparameters}
We initialize the SCGR-augmented GNNs using a \texttt{coupling\_strength} of \( \epsilon=1.0 \) and perform spectral alignment with 100 Fiedler iterations. The auxiliary group modulus \( n \) is selected dynamically via a \texttt{smallest\_prime} strategy, and all models are trained using the Adam optimizer with a consistent learning rate and weight decay.

\subsubsection{Effective Resistance: Benchmark Datasets}

\textbf{\textit{Table \ref{tab:er1}}} represents the ER values for the benchmark datasets. \textit{SCGR} on every dataset attains lower Effective Resistance illustrating better connectivity across graph nodes.

\begin{table*}[h]
\centering
\caption{Effective resistance on benchmark datasets. Percentages in parentheses indicate the reduction in effective resistance relative to the corresponding baseline model.}
\label{tab:er1}
\footnotesize
\renewcommand{\arraystretch}{1.25}
\setlength{\tabcolsep}{5pt}
\begin{tabular}{lccccc}
\toprule
\textbf{Model} 
& \textbf{MUTAG} 
& \textbf{PROTEINS} 
& \textbf{IMDB-BINARY} 
& \textbf{COLLAB} 
& \textbf{ENZYMES} \\
\midrule

GCN 
& $15562 \pm 4098$ 
& $13171 \pm 6650$ 
& $20220 \pm 4351$ 
& $5758 \pm 7127$ 
& $10475 \pm 5842$ \\

GIN 
& $19159 \pm 4698$ 
& $12609 \pm 1662$ 
& $22540 \pm 7967$ 
& $3566 \pm 1404$ 
& $13278 \pm 5662$ \\

\midrule
\midrule

\textbf{GCN + SCGR} 
& \shortstack{$10035 \pm 4339$\\
{\scriptsize $\boldsymbol{\downarrow}\,\mathbf{35.5\%}$}} 
& \shortstack{$12332 \pm 4662$\\
{\scriptsize $\boldsymbol{\downarrow}\,\mathbf{6.4\%}$}} 
& \shortstack{$13091 \pm 3324$\\
{\scriptsize $\boldsymbol{\downarrow}\,\mathbf{35.3\%}$}} 
& \shortstack{$1664 \pm 149$\\
{\scriptsize $\boldsymbol{\downarrow}\,\mathbf{71.1\%}$}} 
& \shortstack{$8845 \pm 748$\\
{\scriptsize $\boldsymbol{\downarrow}\,\mathbf{15.6\%}$}} \\

\midrule

\textbf{GIN + SCGR} 
& \shortstack{$15072 \pm 6654$\\
{\scriptsize $\boldsymbol{\downarrow}\,\mathbf{21.3\%}$}} 
& \shortstack{$11992 \pm 1384$\\
{\scriptsize $\boldsymbol{\downarrow}\,\mathbf{4.9\%}$}} 
& \shortstack{$14556 \pm 4448$\\
{\scriptsize $\boldsymbol{\downarrow}\,\mathbf{35.4\%}$}} 
& \shortstack{$2483 \pm 983$\\
{\scriptsize $\boldsymbol{\downarrow}\,\mathbf{30.4\%}$}} 
& \shortstack{$11116 \pm 3647$\\
{\scriptsize $\boldsymbol{\downarrow}\,\mathbf{16.3\%}$}} \\

\bottomrule
\bottomrule

\end{tabular}
\end{table*}

SCGR reduces effective resistance across all benchmark configurations. The reductions range from approximately 5\% to 71\%, with reductions of 15\% or more in eight of the ten model–dataset configurations. The improvement is comparatively modest on PROTEINS, where SCGR reduces effective resistance by 6.4\% for GCN and 4.9\% for GIN due  to its relatively simple graph topology and less severe connectivity bottlenecks.

The percentage reduction in effective resistance is computed relative to the corresponding baseline model. Specifically, $\mathrm{ER}{\text{baseline}}$ denotes the effective resistance obtained using the original GCN or GIN architecture, while $\mathrm{ER}{\text{SCGR}}$ denotes the value obtained after incorporating SCGR. A positive percentage therefore indicates that SCGR lowers effective resistance, with a larger percentage representing a greater improvement in graph connectivity.

\begin{equation}
\text{ER Reduction}~(\%)
=
\frac{
\mathrm{ER}_{\text{baseline}}
-
\mathrm{ER}_{\text{SCGR}}
}{
\mathrm{ER}_{\text{baseline}}
}
\times 100.
\label{eq:er_reduction}
\end{equation}


\subsubsection{Effective Resistance: Stochastic Block Models}

We evaluate \textit{SCGR} using Stochastic Block Models (SBM) with \(n=1,000\) nodes and \(K=50\) communities. By generating synthetic graphs with 1,000 nodes partitioned into 50 communities, we systematically vary the internal and external edge probabilities to simulate a wide range of graph modularities \(p_{in} : 0.1-0.4 \) and \(p_{out} : 10^{-3}-10^{-1.5}\) to simulate diverse modularity regimes. Each 2-layer GCN uses 64-dimensional random features, ReLU activation, and Adam optimization \(lr=0.01\), weight decay \(= 5\times10^{-4}\) over 80 epochs. \textit{SCGR} augmentation employs a d = 4 Schreier expander, coupling strength \(\epsilon=1.0\), and 100-iteration Fiedler mapping. This controlled setup isolates topological bottlenecks, demonstrating that SCGR's hierarchical bypasses consistently reduce effective resistance and recover accuracy.

We evaluate the structural impact of our expander overlay by computing the average Effective Resistance (ER) using the pseudo-inverse of the graph Laplacian, further validating the results through a Schur complement projection of the augmented Laplacian. 

The results, as shown in the 
\textbf{Table \ref{tab:sbm_results}} and \textbf{Figure \ref{fig:erm}} reveal that SCGR consistently provides alternative low-resistance pathways—bounded theoretically by \(2/(d\gamma)\). This defines significant accuracy gains in low-modularity regimes where the GCN baseline is otherwise severely limited by structural bottlenecks.

\begin{table}[h]
\centering
\caption{GCN accuracy and effective resistance (ER) on SBM graphs.}
\label{tab:sbm_results}
\begin{tabular}{lccc|ccc}
\toprule
\textbf{Mod.} & Base Acc & SCGR Acc & $\Delta_{Acc}$ & Base ER & SCGR ER &  $\Delta_{ER}$\\
\midrule
Low (0.25--0.40)   & 0.4779 & 0.5140 & +7.55\% & 0.582 & 0.341 & -41.4\% \\
Med (0.40--0.70)   & 0.8164 & 0.8219 & +0.79\% & 0.405 & 0.245 & -39.5\% \\
High (0.70--0.85)  & 0.9681 & 0.9731 & +0.20\% & 0.271 & 0.189 & -30.3\% \\
\bottomrule
\end{tabular}
\end{table}

\begin{figure}[ht] 
    \centering
    \includegraphics[width=0.8\linewidth]{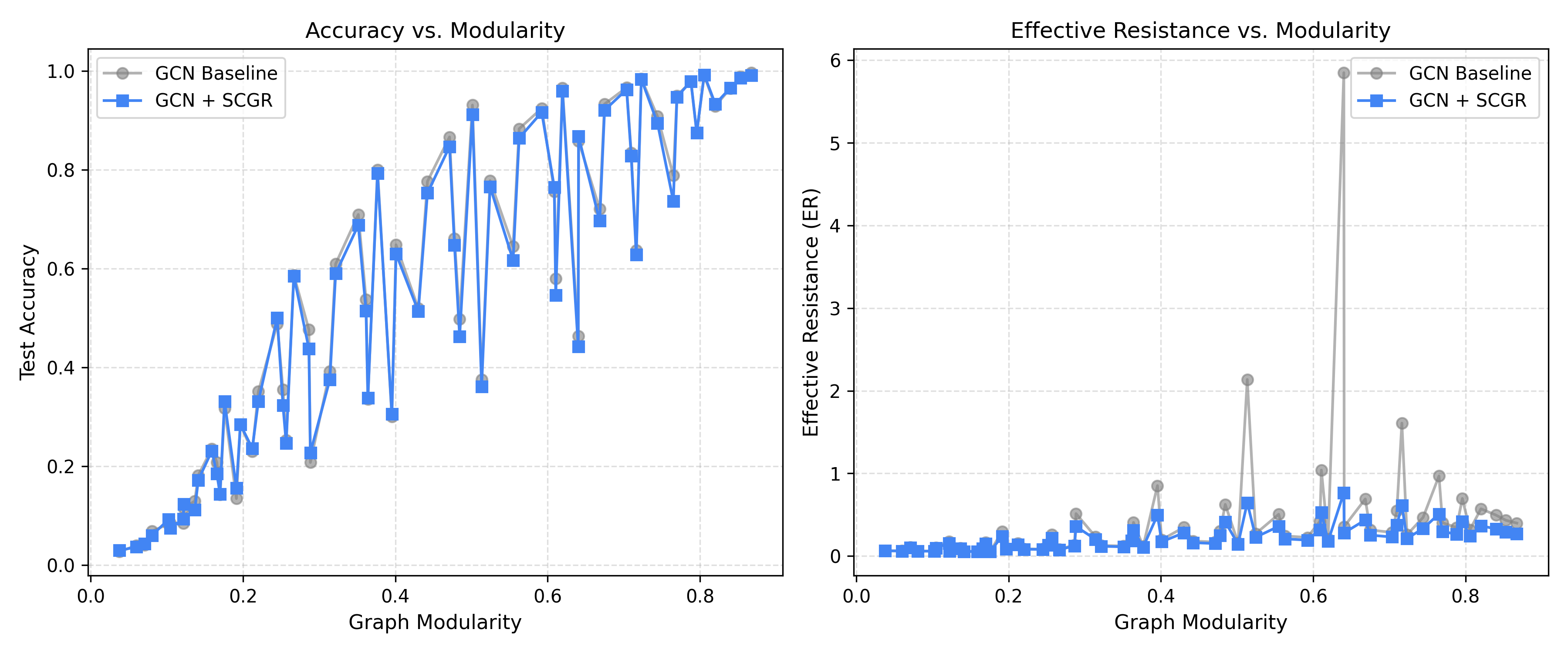}
    \caption{Comparison of effective resistance across different SBM modularities.}
    \label{fig:erm}
\end{figure}

\subsubsection{OGBG - Molhiv, OGBG - Molpcba}

For, molecular prediction task, \textit{SCGR} is assessed on the \textbf{\textit{Open Graph Benchmark - Graph Prediction}} \textit{Ogbg-Molhiv} and \textit{Ogbg-Molpcba} datasets \cite{hu2020open}. The experimental protocol adheres to the implementation and hyperparameter configuration specified by \cite{hu2020open} with \textit{layers: 5, hidden dimension: 300, a dropout: 0.5, and a batch size :64} 

\textit{SCGR} exhibits robust predictive performance while maintaining high structural fidelity. Schreier-coset in both configurations attains highest ROC-AUC score in Molhiv dataset. For the Molpcba dataset, \textit{GCN+SCGR} attains the highest average precision, with \textit{GIN+SCGR} remaining competitive based on the inherent scale and structural complexity of Molpcba.  For GIN on Molpcba, the standard deviation of 0.0767 likely reflects instability in the Fiedler mapping at Molpcba scale, where fixed-budget power iteration may not yield a reliable structural embedding across runs.

\begin{table}[h]
\centering
\caption{Performance comparison on OGBG-MOLHIV and OGBG-MOLPCBA.}
\label{tab:molhiv_ppa}
\renewcommand{\arraystretch}{0.5}
\setlength{\tabcolsep}{2pt}
\begin{tabular}{lcc}
\toprule
Model & OGBG-MOLHIV & OGBG-MOLPCBA \\
\cmidrule(lr){2-2} \cmidrule(lr){3-3}
& Test ROC-AUC $\uparrow$ & Test AP $\uparrow$ \\
\midrule
GCN & & \\
\quad Baseline & 0.7566 ± 0.0104 & 0.2020 ± 0.0024 \\
\quad + Master Node & 0.7531 ± 0.0128 & -- \\
\quad + FA & 0.7628 ± 0.0191 & -- \\
\quad + FLAG & -- & 0.2115 ± 0.0017 \\
\quad + EGP & 0.7731 ± 0.0081 & -- \\
\quad + CGP & 0.7794 ± 0.0122 & -- \\
\midrule
\textbf{\quad + SCGR} & \(\mathbf{0.7834 \pm 0.0141}\) & \(\mathbf{0.2998 \pm 0.0773}\) \\
\midrule
GIN & & \\
\quad Baseline & 0.7678 ± 0.0183 & 0.2266 ± 0.0028 \\
\quad + Master Node & 0.7608 ± 0.0134 & -- \\
\quad + FA & 0.7718 ± 0.0147 & -- \\
\quad + FLAG & -- & 0.2395 ± 0.0040 \\
\quad + EGP & 0.7537 ± 0.0076 & -- \\
\quad + CGP & 0.7899 ± 0.0090 & -- \\
\midrule
\textbf{\quad + SCGR} & \(\mathbf{0.8322 \pm 0.0322}\) & 0.1988 ± 0.07334 \\
\bottomrule
\end{tabular}
\end{table}

\subsubsection{LRGB - Peptides Struct. Peptides Func}

The \textit{Peptides-Struct and Peptides-Func} datasets, sourced from the \textbf{\textit{Long Range Graph Benchmark}} suite provides challenging molecular property prediction tasks that necessitate modeling long-range dependencies in graph structures. \textit{Peptides-Func} (Classification), is evaluated using Average Precision (AP). \textit{Peptides-Struct} (Regression) predicts functional properties of peptides, measured by the mean absolute error (MAE).

\begin{table*}[t]
  \centering
  \caption{Performance comparison on PEPTIDES-FUNC (Test AP $\uparrow$) and PEPTIDES-STRUCT (Test MAE $\downarrow$)}
  \label{tab:lrgb}
  \setlength{\tabcolsep}{12pt}
  \resizebox{0.85\textwidth}{!}{
    \scriptsize
    \begin{tabular}{lcc}
      \toprule
      \textbf{Model} & \textbf{PEPTIDES-FUNC (Test AP $\uparrow$)} & \textbf{PEPTIDES-STRUCT (Test MAE $\downarrow$)} \\
      \midrule
      GCN               & 0.5029 $\pm$ 0.0058 & 0.3587 $\pm$ 0.0006 \\
      \quad + SDRF      & 0.5041 $\pm$ 0.0026 & 0.3559 $\pm$ 0.0010 \\
      \quad + FoSR      & 0.4534 $\pm$ 0.0090 & 0.3003 $\pm$ 0.0007 \\
      \quad + EGP       & 0.4972 $\pm$ 0.0023 & 0.3001 $\pm$ 0.0013 \\
      \quad + CGP       & 0.5106 $\pm$ 0.0014 & 0.2931 $\pm$ 0.0006 \\
      \midrule
      \quad + \textbf{SCGR} & \(\mathbf{0.5301 \pm 0.0010}\) & \(\mathbf{0.2886 \pm 0.0010}\) \\
      \midrule
      GIN               & 0.5124 $\pm$ 0.0055 & 0.3544 $\pm$ 0.0014 \\
      \quad + SDRF      & 0.5122 $\pm$ 0.0061 & 0.3515 $\pm$ 0.0011 \\
      \quad + FoSR      & 0.4584 $\pm$ 0.0079 & 0.3008 $\pm$ 0.0014 \\
      \quad + EGP       & 0.4926 $\pm$ 0.0070 & 0.3034 $\pm$ 0.0027 \\
      \quad + CGP       & 0.5159 $\pm$ 0.0059 & 0.2910 $\pm$ 0.0011 \\
      \midrule
      \quad + \textbf{SCGR} & \(\mathbf{0.5849 \pm 0.0110}\) & \(\mathbf{0.2642 \pm 0.0020}\) \\
      \bottomrule
    \end{tabular}
  }
\end{table*}

\textit{\textbf{Table \ref{tab:lrgb}}} demonstrates SCGR accomplished highest scores in both parameters across both peptide prediction tasks. Schreier Coset delivered a substantial improvement over the strongest baseline, with particularly notable gains when combined with GIN \(:+13.4\%\) on \textit{Peptides-Func} and a \(-9.2\%\) error reduction on \textit{Peptides-Struct}. Even with GCN, \textit{SCGR} outperforms all competing rewiring methods. These consistent improvements across both architectures and tasks validate its effectiveness in enabling GNNs to capture the long-range molecular interactions critical for accurate peptide property prediction, where traditional message passing approaches struggle due to limited receptive fields over-squashing bottlenecks. 

\subsection{Ablation Studies}

\subsubsection{Node Classification}
We evaluate the performance of \(SL(2,\mathbb{Z}_n)\) against a random 4-regular baseline to test the necessity of algebraic symmetry. To isolate the impact of spectral alignment, we contrast Fiedler-based mapping with random node permutations. Finally, we conduct a sensitivity analysis on the coupling strength \(\epsilon \in [0.1,2.0]\).These are performed on Amazon Computers and Amazon Photo.

Results on both Amazon Computers and Amazon Photo yield significant improvements in accuracies. The "Random 4-regular" augmentation consistently performs the worst across both datasets, emphasizing that expansion alone is insufficient. The peak performance at \(\epsilon=1.0 \)suggests that for certain node-level tasks, the SCGR provides critical contextual information that local neighborhoods alone cannot capture.

\begin{table}[ht]
\centering
\caption{Ablation for Node Classification benchmark}
\label{tab:ablation_node_amazon_series}
\small
\begin{tabular}{l | c | c}
\toprule
\textbf{Configuration} & \textbf{Amazon Computers} & \textbf{Amazon Photo} \\
\midrule
GCN (Baseline)          & 0.9058 $\pm$ 0.0028 & 0.9120 $\pm$ 0.0120 \\
+ Random 4-reg          & 0.8987 $\pm$ 0.0070 & 0.8995 $\pm$ 0.4170 \\
+ SCGR (Random Map)     & 0.9028 $\pm$ 0.0050 & 0.9110 $\pm$ 0.3893 \\
\midrule
+ SCGR ($\epsilon=0.1$) & 0.9029 $\pm$ 0.0036 & 0.9227 $\pm$ 0.1040 \\
+ SCGR ($\epsilon=0.5$) & 0.9023 $\pm$ 0.0049 & 0.9310 $\pm$ 0.8830 \\
+ SCGR ($\epsilon=1.0$) & \textbf{0.9031 $\pm$ 0.0062} & \textbf{0.9400 $\pm$ 0.0026} \\
+ SCGR ($\epsilon=2.0$) & 0.9024 $\pm$ 0.0051 & 0.9329 $\pm$ 0.1870 \\
\bottomrule
\end{tabular}
\end{table}


\subsubsection{Graph Classification}
Using the same ablation metrics, we perform tests on Mutag, Enzymes, Proteins, Collab, Imdb domains to ensure the robustness of SCGR design.

The analysis depicts that SCGR \(\epsilon=1.0\) consistently optimizes the bottleneck-performance trade-off. On COLLAB, we achieve a massive \(+8.8\% \)accuracy boost and a \(71.1\%\) reduction in effective resistance. Across all datasets, the Random Map version frequently degrades accuracy, proving that spectral alignment is vital for useful signal propagation. While Random 4-regular expansion often yields lower absolute resistance, it consistently performs poorly in accuracy. Our findings demonstrate that \(\epsilon=1.0\), group theoretic \(SL(2,\mathbb{Z}_n)\) provide the ideal balance for global communication without washing out essential local properties.

\begin{table}[ht]
\centering
\caption{Ablation Study for Graph Classification Benchmark}
\label{tab:ablation_master}
\footnotesize
\setlength{\tabcolsep}{2pt}
\begin{tabular}{l | cc | cc | cc | cc | cc}
\toprule
& \multicolumn{2}{c|}{\textbf{MUTAG}} & \multicolumn{2}{c|}{\textbf{ENZYMES}} & \multicolumn{2}{c|}{\textbf{PROTEINS}} & \multicolumn{2}{c|}{\textbf{COLLAB}} & \multicolumn{2}{c}{\textbf{IMDB-B}} \\
\textbf{Configuration} & \textbf{Acc.} & \textbf{$R_{\text{eff}}$} & \textbf{Acc.} & \textbf{$R_{\text{eff}}$} & \textbf{Acc.} & \textbf{$R_{\text{eff}}$} & \textbf{Acc.} & \textbf{$R_{\text{eff}}$} & \textbf{Acc.} & \textbf{$R_{\text{eff}}$} \\
\midrule
GCN (Baseline)          & 73.160 & 15562 & 53.750 & 10475  & 69.200 & 13171 & 68.800 & 5758.7 & 63.400 & 20220 \\
+ Random 4-reg          & 72.630 & 7201  & 51.170 & 6120.8 & 69.200 & 7200.0 & 69.800 & 2161.4 & 58.800 & 10098 \\
+ SCGR (Random Map)     & 71.840 & 15751 & 53.920 & 11226  & 69.420 & 14887 & 71.000 & 2822.2 & 58.100 & 12959 \\
\midrule
+ SCGR ($\epsilon=0.1$) & 75.260 & 14703 & 52.420 & 8709.8 & 68.260 & 15340 & 74.400 & 3086.7 & 60.750 & 15472 \\
+ SCGR ($\epsilon=0.5$) & 75.530 & 13798 & 51.920 & 8837.2 & 70.800 & 15772 & 71.600 & 2657.2 & \textbf{62.050} & 13321 \\
+ SCGR ($\epsilon=1.0$) & \textbf{77.890} & \textbf{10035} & \textbf{53.830} & \textbf{8845.07} & \textbf{72.590} & \textbf{12332} & \textbf{77.600} & \textbf{1664.0} & 61.600 & \textbf{13091} \\
+ SCGR ($\epsilon=2.0$) & 73.420 & 15015 & 49.420 & 10581  & 68.660 & 14034 & 72.400 & 4453.6 & 57.300 & 13027 \\
\bottomrule
\end{tabular}
\end{table}

\end{document}